\documentclass[runningheads]{llncs}
\usepackage{graphicx}
\usepackage{comment}
\usepackage{booktabs}       
\usepackage{multirow}

\usepackage{subfigure}

\usepackage[width=122mm,left=12mm,paperwidth=146mm,height=193mm,top=12mm,paperheight=217mm]{geometry}

\usepackage{bbding}

\usepackage{epsfig}
\usepackage{graphicx}
\usepackage{amsmath}
\usepackage{amssymb}
\usepackage{mathtools}
\usepackage{textcomp}
\usepackage{xcolor}

\usepackage[utf8]{inputenc} 
\usepackage[T1]{fontenc}    

\usepackage{booktabs}       
\usepackage{amsfonts}       
\usepackage{nicefrac}       
\usepackage{microtype}      
\usepackage{enumitem}
\DeclareMathOperator*{\argmax}{arg\,max}

\newcommand{\tabincell}[2]{\begin{tabular}{@{}#1@{}}#2\end{tabular}} 

\usepackage{multirow}
\usepackage{makecell}

\newtheorem{thm}{Theorem}

\usepackage{color}
\definecolor{mydarkblue}{rgb}{0,0.08,0.45}
\definecolor{darkgreen}{rgb}{0.0, 0.5, 0.0}

\begin{document}
\pagestyle{headings}
\mainmatter
\def\ECCVSubNumber{2148}  

\title{Novel Human-Object Interaction Detection via \\ Adversarial Domain Generalization} 

\titlerunning{Novel HOI Detection via Adversarial Domain Generalization}
%
\author{
Yuhang Song\textsuperscript{$1$}
\thanks{Work done while interning at Microsoft Research AI lab.}
\and Wenbo Li\textsuperscript{$2$}
\and Lei Zhang\textsuperscript{$3$} 
\and Jianwei Yang\textsuperscript{$3$} 
\and Emre Kiciman\textsuperscript{$3$} 
\and \\
 Hamid Palangi\textsuperscript{$3$} 
\and Jianfeng Gao\textsuperscript{$3$}
\and C.-C. Jay Kuo\textsuperscript{$1$} 
\and Pengchuan Zhang\textsuperscript{$3$}
}
\authorrunning{Y. Song, W. Li, L. Zhang et al.}
%
\institute{
 \textsuperscript{$1$}University of Southern California
\hspace{10mm}
 \textsuperscript{$2$}Samsung Research America AI Center
 \\
 \textsuperscript{$3$}Microsoft Corporation
}
\maketitle

\begin{abstract}
\vspace{-2mm}
We study in this paper the problem of novel human-object interaction (HOI) detection, aiming at improving the generalization ability of the model to unseen scenarios. The challenge mainly stems from the large compositional space of objects and predicates, which leads to the lack of sufficient training data for all the object-predicate combinations. As a result, most existing HOI methods heavily rely on object priors and can hardly generalize to unseen combinations. To tackle this problem, we propose a unified framework of adversarial domain generalization to learn object-invariant features for predicate prediction. To measure the performance improvement, we create a new split of the HICO-DET dataset, where the HOIs in the test set are all unseen triplet categories in the training set. Our experiments show that the proposed framework significantly increases the performance by up to 50\% on the new split of HICO-DET dataset and up to 125\% on the UnRel dataset for auxiliary evaluation in detecting novel HOIs.

\end{abstract}
\section{Introduction}
\noindent Over the past few years, rapid progress has been made in visual recognition tasks, but image understanding also calls for visual relationship detection, i.e., detection of <subject, predicate, object> triplets in the image. While some efforts have been made to detect the general relationships between different objects~\cite{lu2016visual,liang2017deep,dai2017detecting,xu2017scene,zhang2017visual,liang2018visual,zellers2018neural,yang2018graph}, one particularly important class of visual relationship detection requiring further study is the Human-Object Interaction (HOI) detection, where only relations with human subjects are of interest~\cite{chao2015hico,gupta2015visual,chao2018learning,gao2018ican,li2018transferable,qi2018learning,fgkioxari2018detecting}. 

%
\begin{figure}[t]
\centering
\small
\setlength{\tabcolsep}{1pt}
\begin{tabular}{ccc}
 \includegraphics[width=.33\linewidth]{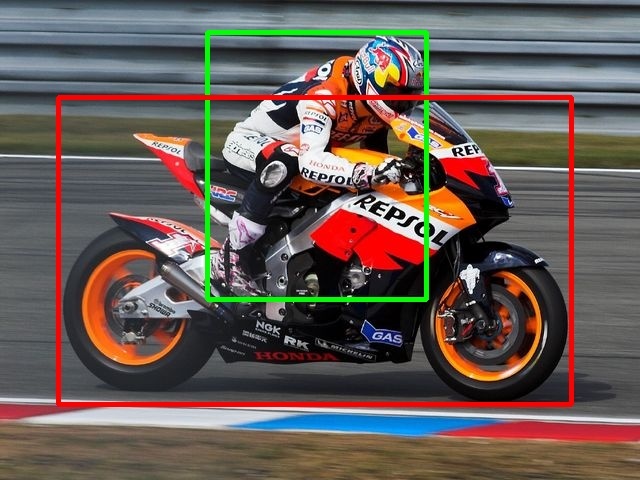}&
 \includegraphics[width=.33\linewidth]{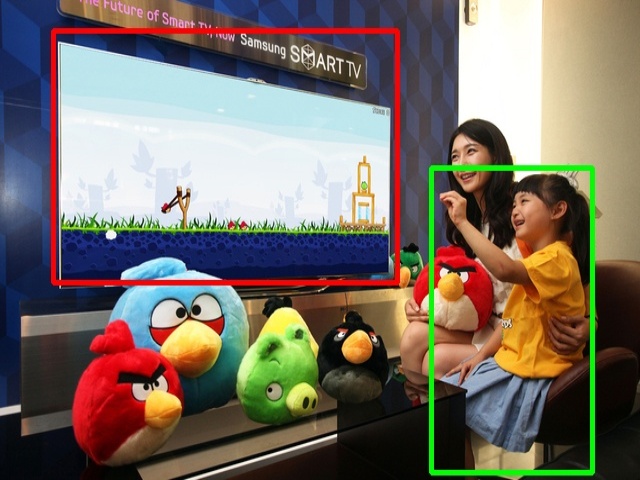} & 
 \includegraphics[width=.33\linewidth]{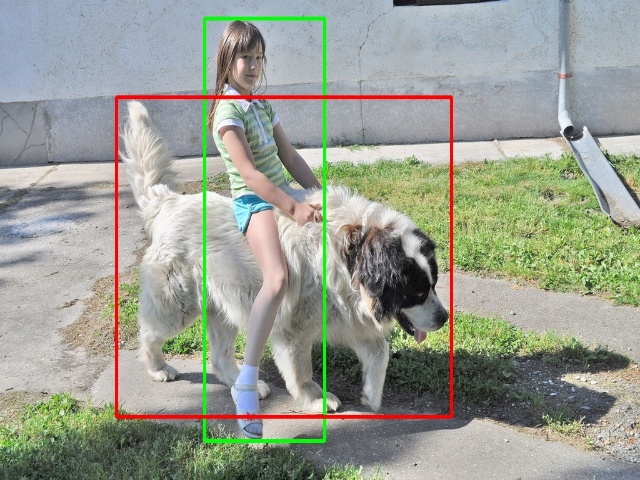} \\
 Annotation: \textcolor{darkgreen}{ride (motorcycle)} &
 Annotation: \textcolor{darkgreen}{watch (tv)} &
 Prediction: \textcolor{darkgreen}{ride (dog)} \\
\end{tabular}
\vspace{-3mm}
\caption{Novel relationship detection. \textbf{Green box}: subject. \textbf{Red box}: object. First two images are fro training set while the last image contains an unseen triplet from test set.}
\vspace{-5pt}
\label{fig:teaser_novelrelationdetection}
\end{figure} 



A long-standing problem in both HOI detection and visual relationship detection is the long-tail problem, where specific predicates dominate the triplet instances for most of the object categories. Fig.~\ref{fig:hico_longtail} shows the distribution of both the triplet categories and the object categories given the predicate ``horse'' in the HICO-DET dataset. In both cases, a small number of categories dominate the training instances, allowing a learned model to rely on a frequency prior rather than learning the relationship feature itself. The same conclusion is made in Visual Genome dataset~\cite{krishna2017visual} by Zellers \emph{et al.}~\cite{zellers2018neural}, where they point out that the frequency prior is a main barrier for visual relationship detection. 

Collecting a balanced dataset is a simple approach to tackle this problem.  However, if we have $N$ predicates and $M$ objects, the possible combination of the triplets is $MN$.  It is difficult to collect all those possible combinations due to the infrequency of relationships. For the relationship detection task, the long-tail problem, combinatorial problem and frequency prior barrier are closely related in the sense that the distribution of the triplet categories are extremely imbalanced in the large compositional space of triplet categories.

Motivated by these observations, we focus on the {\bf novel} HOI detection problem~\cite{shen2018scaling}, where the predicate-object combinations in test set are never seen in the train set. As shown in Fig.~\ref{fig:teaser_novelrelationdetection}, the training and test set share the same predicate categories, but the combinations of predicate and object categories in the test set are unseen. This task is challenging because the model is required to learn object-invariant predicate features and generalize to unseen interactions, which are able to be further applied to downstream tasks. 

\begin{figure}[!ht]
\vspace{-20pt}
	\centering
	\includegraphics[width=1.\linewidth]{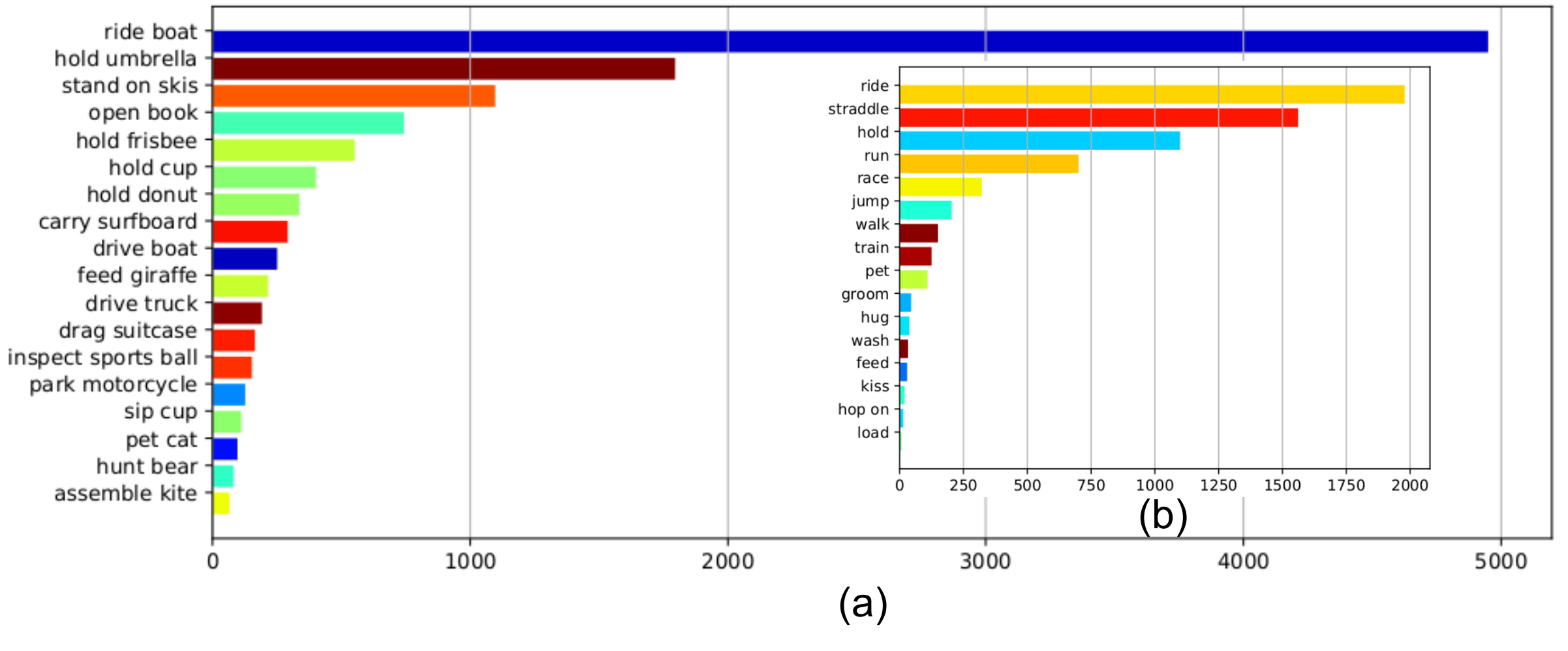}
	\vspace{-10pt}
	\caption{Number of instances in the HICO-DET dataset for each (a) HOI category (b) predicate category of "horse".}
	\vspace{-5pt}
\label{fig:hico_longtail}
\end{figure}


Our first contribution is to create a new benchmark dataset for the novel HOI detection task, based on the images and annotations from the HICO-DET dataset~\cite{chao2018learning} and the UnRel dataset~\cite{peyre2017weakly}. The new benchmark dataset avoids the overlapping of the triplet categories in the training set, validation set and test set. This new benchmark contains an additional evaluation set from UnRel dataset~\cite{peyre2017weakly},  highlighting its instances with unusual scenes.

Our second contribution is to propose a unified adversarial domain generalization framework, which can serve as a plug-in module for existing models to improve their generalization ability. We instantiate both conditional and unconditional methods within the framework and build its relationship with previous methods. Experiments on HICO-DET and Unrel dataset show that our proposed adversarial training can get uniformly significant improvement on all metrics. Our work shows promising results of adversarial domain generalization in conquering the combinatorial prediction problem in real-world applications.

\section{Related Work}

\noindent\textbf{Visual relationship detection and human-object interaction detection}
Visual relationship detection~\cite{sadeghi2011recognition,krishna2017visual} has seen a surge of interests recently~\cite{lu2016visual,dai2017detecting,xu2017scene,li2017scene,yang2018graph,zhang2019graphical,tang2020unbiased}, due to the great success of deep learning on other 2D vision tasks. However, the current SOTA performance~\cite{tang2020unbiased} on the Visual Genome dataset~\cite{krishna2017visual} is still marginally better than the frequency prior baseline~\cite{zellers2018neural}, due to the serious long-tail problem in visual relationship detection. Detecting HOI~\cite{gupta2015visual,chao2015hico,chao2018learning} requires a deeper understanding of the scenario in images and presents unique challenges for visual relationship detection, as human actions mostly relate to semantic verbs, which are more challenging.
\cite{fgkioxari2018detecting,wan2019pose,zhou2019relation,gupta2019no} exploited the cue of human appearance by making use of detected human key points. 
\cite{gao2018ican,wang2019deep} employed novel attention modules to exploit the contextual information. \cite{qi2018learning} designed a graph parsing network to incorporate structural knowledge. Li \emph{et al.}~\cite{li2019transferable} applied interactiveness prior to boost the performance, which is learned across multiple datasets.

\noindent\textbf{Long-tail or novel human-object interaction detection}
\cite{shen2018scaling} formulates the novel HOI detection as a zero-shot learning problem, and propose to detect the predicates and objects separately. \cite{shen2018scaling} focuses on HOI detection and train the predicate and object detector jointly. Our paper focuses on predicate classification and assumes given ground-truth/detected objects. Different from \cite{shen2018scaling}, we introduce the adversarial training to learn object-independent features for predicate classification, so that the model can robustly generalize to unseen triplet combinations. There are other approaches towards long-tail or novel HOI detection, such as \cite{xu2019learning} making use of external knowledge and \cite{peyre2018detecting} obtaining visual-phrase embeddings of unseen relations from transfer by seen analogies.


\noindent\textbf{Adversarial domain generalization (ADG)} 
Inspired by generative adversarial networks\cite{goodfellow2014generative}, adversarial domain adaption methods, e.g., \cite{ganin2014unsupervised,tzeng2015simultaneous,tzeng2017adversarial}, have been successfully embedded into deep networks to learn transferable features to reduce distribution discrepancy between the source and target domains. In contrast, domain generalization (DG)\cite{khosla2012undoing,ghifary2016scatter,muandet2013domain,xu2014exploiting,ghifary2015domain} aims to learn a model from (multiple) source domains and generalize it to unseen target domains, and thus does not require unlabeled data from the target domains. Recently, \cite{li2018deep} propose a conditional invariant deep domain generalization method to learn a domain-invariant representation by making the learned representations on different domains indistinguishable. Our work generalizes \cite{li2018deep} with generalized discrepancy measures and simpler solutions for real-world large-scale training. As a result, our method can deal with real-world challenging DG problems with large number of source domains and huge variations of label distributions across source domains.
\section{Problem Statement}
\subsection{Problem Formulation}\label{sec:problemformulation}
\noindent Suppose the training set and test set are represented by $\mathcal{D}_{train} = \{(I_i, (b_S)_i, (b_O)_i, \\S_i, O_i, P_i)\}$ and $\mathcal{D}_{test} = \{(I_j, (b_S)_j, (b_O)_j, S_j, O_j, P_j)\}$, where $b_S$ and $b_O$ are the bounding boxes of subjects and objects, and $I$, $S$, $O$, $P$ denote the images, subject labels, object labels and predicate labels. The novel HOI detection task is defined by the constraint that there are no overlapping combinations of $P$ and $O$. The goal of HOI detection is to learn a function $\mathcal{F}: I \rightarrow \{b_S, b_O, O, P\}$.
This contains two steps, object detection and predicate detection, making the problem more complex and difficult to analyze. In this paper, we focus on predicate prediction problem to learn the object-invariant features, where we formulate it as: $\mathcal{F}: \{I, b_S, b_O\} \rightarrow P$.


\subsection{Dataset Creation}\label{sec:datasetcreation}
\noindent The most commonly used datasets for HOI detection are V-COCO~\cite{gupta2015visual} and HICO-DET~\cite{chao2018learning}. As V-COCO is relatively small,
it is insufficient for evaluation of novel HOI detection. Therefore, we primarily use the HICO-DET dataset for our experiments and evaluation, with 600 HOI categories and over 150K annotated instances of human-object pairs. We extract 117 predicates and 80 object categories from the 600 HOI categories to evaluate the HOI detection performance of the model on unseen <human, predicate, object> triplets.


However, the original split of the HICO-DET was not designed to verify the effectiveness of the proposed models when being transferred to unseen predicate-object pairs. With the original train/test split, \cite{shen2018scaling} use part of predicate-object combinations in the train set and the other part of predicate-object combinations in the test set to set up the novel HOI detection task. However, this approach discards most of the data in the original test split and results in a very small novel-HOI test set and thus large fluctuation of evaluation metrics. Moreover, it is not clear how the validation set is set up and how the hyperparameters is tuned in \cite{shen2018scaling}.
Therefore, we create a new split of the HICO-DET based on its images and annotations, based on the principle that none of the <human, predicate, object> triplet categories in the test set should exist in the training set. 
We collect the triplet instances in the whole dataset and then divide them into 90\% {\it training} and 10\% {\it test} sets without overlapping triplet categories. \footnote{We also ensure the triplet instances in the training set and the test set are from different image sets.}
We further divide the training set into 7/9, 1/9 and 1/9 for {\it training} and {\it trainval} and {\it testval} splits, where the {\it training} and {\it trainval} splits share the same triplet categories and the {\it testval} split has no overlapping triplet categories with training and {\it trainval}. The combined {\it trainval} and {\it testval} splits are used as the validation set for hyperparameter tuning. 
The statistics of this new split are listed in Table~\ref{tb:newsplitstatistics}. 


\begin{table}[t]
\footnotesize
\setlength{\tabcolsep}{2.5pt} 
\center
\resizebox{.8\linewidth}{!}{
  \begin{tabular}{l c c c c c c c c }
   \toprule
         & \multicolumn{5}{c}{HICO-DET}     
         & \multicolumn{3}{c}{UnRel}     \\
    \cmidrule(r){2-6}
    \cmidrule(r){7-9}
  & train & trainval & testval & test & total & split 1 & split 2 & split 3   \\
    \cmidrule(r){1-9}
   \multirow{2}{*}{}
 \#images  & 31873 & 4357 & 5528 & 5421 & 47179 & 196 & 248 & 494  \\
 \#instances  & 106043 & 15149 & 10408 & 10391 & 141991 & 323 & 396 & 718  \\
\bottomrule
\end{tabular}
}
\vspace{2mm}
\caption{ Statistics of the new splits}
\label{tb:newsplitstatistics}
\vspace{-5mm}
\end{table}

In addition to the HICO-DET dataset, we use the UnRel dataset~\cite{peyre2017weakly} for auxiliary validation for models trained on HICO-DET dataset. The UnRel dataset was originally used for image retrieval and was designed for the detection of rare relations. It contains 1,071 images with annotations and 76 triplet categories. We choose the instances with the same predicate categories as those in the HICO-DET to validate the generalization ability of the model. Based on the type of the chosen predicate and object types, we created three different splits in the UnRel dataset to verify the generalization performance in different levels: \\
\textbf{Split 1:} $\mathcal{D}_{unrel}^{1} = \{(I_{k_1}, (b_S)_{k_1}, (b_O)_{k_1}, S_{k_1}, O_{k_1}, P_{k_1})\}$, $S_{k_1}=human$, $\{O_{k_1}\} \subset \mathcal{D}_{train}$, $\{P_{k_1}\} \subset \mathcal{D}_{train}$. \\ 
\textbf{Split 2:} $\mathcal{D}_{unrel}^{2} = \{(I_{k_2}, (b_S)_{k_2}, (b_O)_{k_2}, S_{k_2}, O_{k_2}, P_{k_2})\}$, $S_{k_2}=human$, $\{P_{k_2}\} \subset \mathcal{D}_{train}$. \\
\textbf{Split 3:} $\mathcal{D}_{unrel}^{3} = \{(I_{k_3}, (b_S)_{k_3}, (b_O)_{k_3}, S_{k_3}, O_{k_3}, P_{k_3})\}$, $\{P_{k_3}\} \subset \mathcal{D}_{train}$. \\
The statistics of the splits are shown in Table~\ref{tb:newsplitstatistics}.

\subsection{Evaluation Metrics}\label{sec:evaluationmetrics}
In this paper, we focus on the predicate detection and use ground-truth object boxes and labels for evaluation. Due to the ambiguity and incompleteness of HOI annotations (e.g., ``ride'' vs ``striddle''), we propose to use recall metrics (which are widely used in visual relationship detection and scene graph generation~\cite{lu2016visual,zhang2017visual,liang2018visual,zellers2018neural,zhang2019graphical,tang2020unbiased}) as follows:  \\
(1) \noindent\textbf{Predicate classification} (PredCls): 
For each human-predict-object triplet in the test set, predict the predicate class given the ground-truth bounding boxes and object label. \\
(2) \noindent\textbf{Predicate detection} (PredDet): 
For each image in the test set, detect all human-predict-object triplets given the ground-truth bounding boxes and their associated labels. \\ 
We can replace the ground-truth object bounding boxes and labels with detected bounding boxes and labels (from a pretrained or jointly trained object detector), and the PredDet metric becomes the standard SgDet metric in scene graph generation~\cite{lu2016visual,zhang2017visual,liang2018visual,zellers2018neural,zhang2019graphical,tang2020unbiased}. However, the SgDet metric is very sensitive to the object detection performance. Therefore, our metrics use ground-truth object boxes and labels to exclude the factor of object detection performance.

\section{Method}
\noindent In this section, we start with an overview of our baseline in Section~\ref{sec:methodoverview}. Then we present our adversarial domain generalization framework in Section~\ref{sec:adversarialtraining}, instantiate three approaches in Section~\ref{subsec:uncondADG} and \ref{sec:cadvda}, and discuss the relation between our framework and previous DeepC~\cite{li2018deep} in Section~\ref{subsec:ourvsdeepc}. Finally, we introduce the implementation in Section~\ref{sec:architectures_ADG}.

\begin{figure}[t]
	\centering
	\includegraphics[width=1.\linewidth]{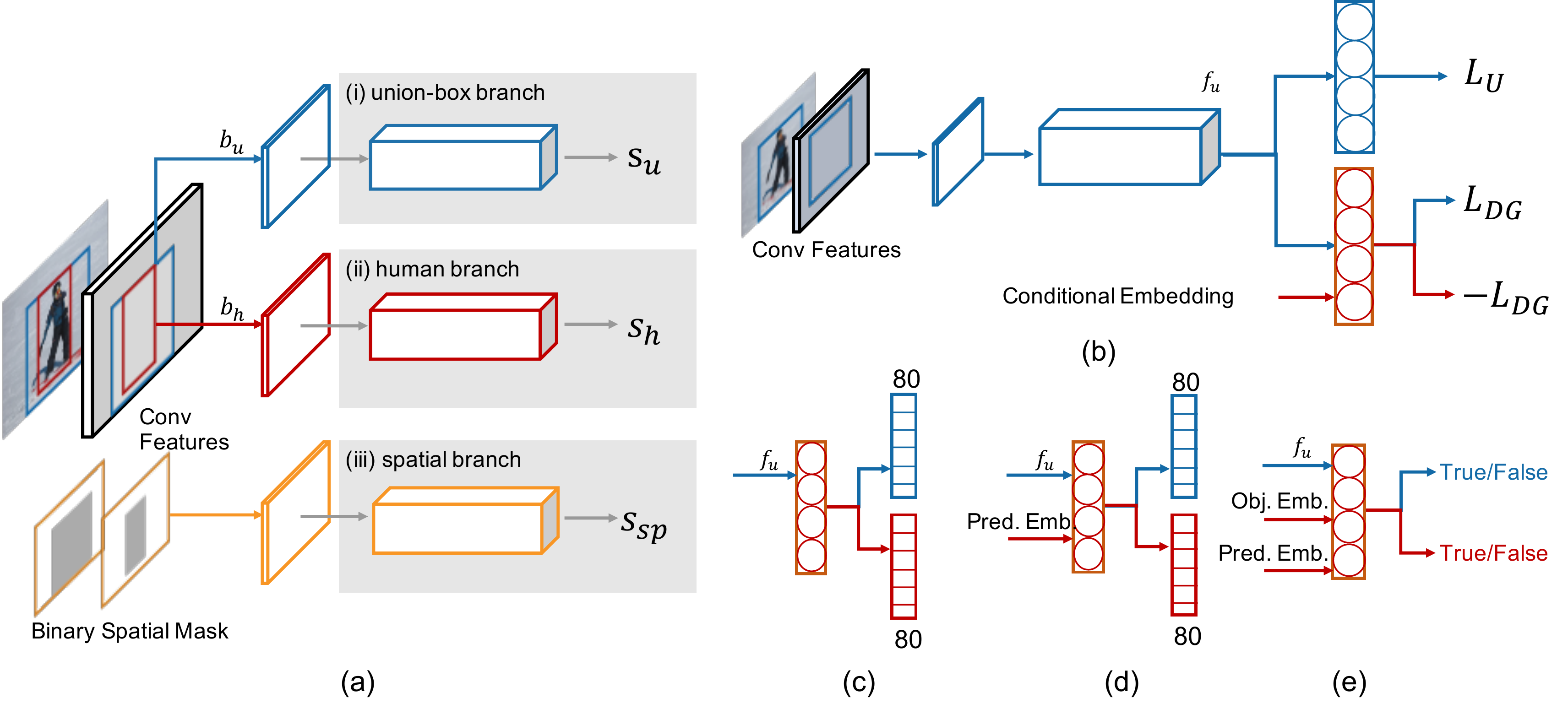}
	\caption{Architectures. (a) Baseline architecture which consists of (i) a union-box branch,(ii) a human branch, and (iii) a spatial branch. (b) the proposed ADG framework for domain generalization. (c) ADG-KLD. (d) CADG-KLD.  (e) CADG-JSD.}
\label{fig:all_frameworks}
\end{figure}

\subsection{Overview}\label{sec:methodoverview}
\noindent Following the conventions in \cite{gao2018ican,fgkioxari2018detecting,li2018transferable}, our baseline model has three branches
to extract different types of visual features for predicate prediction, as shown in Fig.~\ref{fig:all_frameworks}(a). We use the union-box branch rather than the object branch to extract visual features as the union box contains more visual information to capture the \textit{interaction} between human and object. Our proposed domain generalization approaches are \textbf{only} applied to the union-box branch, while keeping the other two branches unchanged for fair comparison.



We denote the human box and union box as $b_h$ and $b_u$. From the three branches, we predict three probabilities of the predicate category, denoted as $s_h$, $s_u$ and $s_{sp}$. 
As the same human-object pair could have multiple predicates as the ground truth labels, we take the predicate prediction as a multi-label classification problem based on a binary sigmoid classifier, and minimizes the cross entropy losses on three branches for each category, denoted as $L_H$, $L_{sp}$, and $L_U$. Then the total loss function is defined as:
\begin{equation}\label{eqn:baseline_loss}
\vspace{-1mm}
L_{baseline} = L_H + L_{sp} + L_U
\vspace{-1mm}
\end{equation}

In the inference time, we rank the scores of each triplet based on the formula $score_{triplet} =  (s_h + s_u) \cdot s_{sp}$.




\subsection{Adversarial domain generalization (ADG)}\label{sec:adversarialtraining}
\noindent To learn a visually grounded relationship feature that can generalize to novel <predicate, object> pairs, the feature should be as object-invariant as possible. In this paper, we focus on learning a visually grounded relationship feature from the union box branch, i.e., $f_u$ in Figure~\ref{fig:all_frameworks} (b), because the spatial branch feature and the human branch feature are both expected to be object-invariant by design.

We view this object-invariant feature learning as a domain generalization (DG) problem, where each object category is  viewed as a separate domain. We aim at learning domain(object)-invariant features for predicting class (predicate category). {\it For a given predicate}, say, "ride", we only have "ride-horse" and "ride-bicycle" in the training data, i.e., training data is only collected from domains "horse" and "bicycle". However, we have unseen pair “ride-dog” in the test data, i.e., we want our model to generalize well to new domain "dog".

This DG problem is extremely challenging, due to three reasons. First, our feature extractor $f_u = F(x)$ is a deep neural network, while nearly all previous DG methods~\cite{khosla2012undoing,ghifary2016scatter,muandet2013domain,xu2014exploiting,ghifary2015domain} are only tested on linear feature extractor. Very recent work~\cite{li2018deep} showed promising results on domain generalization on deep features. Second, the number of domains is large, i.e., 80 object categories in our case. Third, there are huge variations in predicate class distribution across domains. For example, the predicate class distributions of domain "horse" (see Figure~\ref{fig:hico_longtail}(b)) and domain "cup" are largely different. \cite{li2018deep} is only tested on problems with at most 5 domains and does not work with large variation in class distribution across domains, as mentioned in their paper and shown in our results. 

In the rest of this section, we propose a general framework for adversarial domain generalization (ADG).
We introduce a DG regularization $L_{DG}$ into the training, i.e.,
\begin{equation}\label{eqn:adversialDA}
\vspace{-1mm}
L_{total} = L_H + L_{sp} + L_U + \lambda L_{DG},
\vspace{-1mm}
\end{equation}
which effectively inject an inductive bias in the training process to learning domain(object)-invariant features. $L_{DG}$ involves divergence between high-dimensional distributions, so we introduce discriminators to estimate it and perform alternative adversarial training to minimize the total loss function~\eqref{eqn:adversialDA}; see Figure~\ref{fig:all_frameworks}. 

We denote $M$ the number of domains (i.e., object categories) and $K$ the number of classes (i.e., predicate categories). In our task, $M=80$ and $K=117$.


\subsubsection{Unconditional adversarial domain generalization (ADG-KLD)}\label{subsec:uncondADG}
\noindent A first attempt is to enforce the invariance of extracted feature distributions across domains, i.e., $P(f_{u}| obj_i) = P(f_{u}| obj_j)$ for any two different object categories $obj_i$ and $obj_j$. For example, it is expected that the union-box feature distributions of domain "horse" and "elephant" are similar, because they share similar interactions with human. This shared feature space is expected to be more independent from the object categories and more applicable to an unseen domain "donkey".  

Enforcing the mutual similarity between two domains is equivalent to enforce the similarity to the pooled feature distribution for each domain:
\begin{equation}\label{eqn:unconditionDA}
\vspace{-1mm}
    P(f_{u}| obj_i) = P(f_{u}) \quad \forall i \in [M],
\vspace{-1mm}
\end{equation}
where $P(f_{u})$ is the pooled feature distribution
\begin{equation}\label{eqn:advdadist}
\vspace{-1mm}
	P(f_{u}) = \sum_{i\in [M]} \alpha_{i} P(f_{u}| obj_i), \quad \sum_{i\in [M]} \alpha_{i} = 1
\vspace{-1mm}
\end{equation} 
and $\alpha_{i}$ is the relative importance of each domain. In this paper, we choose $\alpha_{i} = N_i/N$, i.e., the fraction of data in domain $i$.

We use adversarial training to enforce this distribution match. Specifically, we introduce a discriminator $D: f_{u} \to [0,1]^{M}$ that tries to classify the domain (object) based on the union-box feature, while the feature extractor $F$ is trying to confuse the discriminator. Formally, they play a mini-max game as follows: 
\begin{equation}\label{eqn:unconditionDA_Di}
\vspace{-1mm}
    \min_{F} \max_{D} \sum_{i\in [M]} \alpha_{i} E_{f \in P(f_{u}| obj_i)}[\log D_i(f)].
\vspace{-1mm}
\end{equation}
In practice, we optimize this with stochastic gradient descent (SGD). For each sample $x$ (from domain $obj(x)$), we update $D$ and $F$ with the minimax loss:
\begin{equation}\label{eqn:unconditionDA_Dierm_sgd}
\min_{F} \max_{D} \log D_{obj(x)}(F(x)),
\end{equation}
where $w_i = N_i/N$ is used to reduce \eqref{eqn:unconditionDA_Di} to \eqref{eqn:unconditionDA_Dierm_sgd}.

Assume infinite capacity of the discriminator $D$, the maximum of \eqref{eqn:unconditionDA_Di} is a weighted summation of KL divergence between distributions:
\begin{equation}\label{eqn:advdakldiv}
\vspace{-1mm}
L_{DG} = KLD := \sum_{i\in [M]} \alpha_{i} KL(P(f_{u}| obj_i) || P(f_{u})).
\vspace{-1mm}
\end{equation} 
We provide the proof in Appendix. Therefore, the adversarial training in \eqref{eqn:unconditionDA_Dierm_sgd} is indeed adding a DG regularization \eqref{eqn:advdakldiv} into the training objective~\eqref{eqn:adversialDA}, effectively enforcing invariance of feature distributions \eqref{eqn:unconditionDA}.

\subsubsection{Conditional adversarial domain generalization}\label{sec:cadvda}
\noindent However, due to class distribution mismatch across different domains, \eqref{eqn:unconditionDA} may not be the invariance we want to achieve in object-invariant predicate detection. For example, the predicate class distributions of domain "horse" and domain "cup" are completely different, and thus it is not reasonable to enforce predicate feature distribution match between these two domains.

By considering the class distribution mismatch, we can instead enforce that the conditional feature distribution is the same over different domains (objects):
\begin{equation}\label{eqn:cadvda}
\vspace{-1mm}
    P(f_{u}| obj_i, pred_k) = P(f_{u}| pred_k) \quad \forall i \in [M], k \in [K],
\vspace{-1mm}
\end{equation}
where the pooled conditional distribution 
\begin{equation}\label{eqn:cadvdadist}
\vspace{-1mm}
\begin{aligned}
P(f_{u}| pred_k) = \sum_{i\in [M]} \alpha_{i}^{(k)} P(f_{u}| obj_i, pred_k)
\end{aligned}
\vspace{-1mm}
\end{equation} 
and $\alpha_{i}^{(k)}$ ($\sum_{i\in [M]} \alpha_{i}^{(k)} = 1$) is the relative importance of different conditional distributions. In this paper, we choose $\alpha_{i}^{(k)} = \frac{N_{i}^{(k)}}{N^{(k)}}$ where $N_{i}^{(k)}$ is the number of training samples in domain (object category) $i$ and with (predicate) class $k$ and $N^{(k)}=\sum_{i\in [M]} N_{i}^{(k)}$ is the number of samples with (predicate) class $k$ in the full training dataset. One can also choose other weights, like $\alpha_{i}^{(k)} \equiv 1/M$ in \cite{li2018deep}, and methods presented below can be still applied with corresponding re-weighting of training samples. However, as we show in Table~\ref{tb:newsplit}, $\alpha_{i}^{(k)} \equiv 1/M$ in \cite{li2018deep} gets very marginal improvement while our weights achieves significant improvement. We provide more discussions on this in Appendix.

In the following, we present two methods to achieve \eqref{eqn:cadvda}, by specifying two different kinds of divergence as the DG regularization $L_{DG}$ in \eqref{eqn:adversialDA}.

\noindent\textbf{KL Divergence (CADG-KLD)}
\noindent In the first method, similar to Section~\ref{subsec:uncondADG}, we enforce \eqref{eqn:cadvda} with the following conditional KL divergence:
\begin{equation}\label{eqn:cadvdakldiv}
\vspace{-1mm}
\begin{aligned}
& L_{DG} = CKLD := \sum_{k\in[K]} \alpha^{(k)} \sum_{i\in [M]} \alpha_{i}^{(k)}  KL(P(f_{u}| obj_i, pred_k) || P(f_{u}| pred_k)) ,
\end{aligned}
\vspace{-1mm}
\end{equation} 
where $\alpha_{i}^{(k)}$ are the weights defined in \eqref{eqn:cadvdadist} and $\alpha^{(k)}$ are weights that balance different classes. To achieve this, we introduce discriminators conditioned on $pred_k$, written as $D(f; pred_k) \in [0,1]^{M}$, which tries to classify the sample's domain (object category). The feature extractor $F$ tries to confuse the discriminator:
\begin{equation}\label{eqn:cadvdakldiv_Di}
\vspace{-1mm}
\begin{aligned}
& \min_{F} \max_{D} \sum_{k\in[K]} \alpha^{(k)} \sum_{i\in [M]} \alpha_{i}^{(k)}  E_{f \in P(f_{u}| obj_i, pred_k)}[\log D_i(f; pred_k)] .
\end{aligned}
\vspace{-1mm}
\end{equation}
We show that in the appendix, assuming infinite capacity of $D$, the maximum of \eqref{eqn:cadvdakldiv_Di} (up to a constant) is indeed $CKLD$ defined in \eqref{eqn:cadvdakldiv}. We propose to use $\alpha^{(k)} = N^{(k)}/N$ and $\alpha_{i}^{(k)} = N_{i}^{(k)}/N^{(k)}$. In practice, we optimize \eqref{eqn:cadvdakldiv_Di} with SGD, in which for each sample $x$ (with label $pred(x)$ and domain $obj(x)$), we update $D$ and $F$ with the following regularization:
\begin{equation}\label{eqn:cadvdakldiv_Derm_sgd}
\min_{F} \max_{D} \log D_{obj(x)}(F(x); pred(x)).
\end{equation}

\noindent\textbf{Jensen-Shannon Divergence (CADG-JSD)}
In the second method, we enforce \eqref{eqn:cadvda} with the following conditional JSD:
\begin{equation}\label{eqn:cadvdajddiv}
\vspace{-1mm}
\begin{aligned}
& L_{DG} = CJSD := \sum_{k\in[K]} \alpha^{(k)} \sum_{i\in [M]} \alpha_{i}^{(k)}  JSD(P(f_{u}| obj_i, pred_k) || P(f_{u}| pred_k)) ,
\end{aligned}
\vspace{-1mm}
\end{equation} 
where $\alpha^{(k)}$ and $\alpha_{i}^{(k)}$ are weights that users specified to balance different terms. We introduce discriminators conditioned on $pred_k$, written as $D(f, obj_i; pred_k) \in [0,1]$. The objective of the discriminator is to distinguish where the features is from the domain specific distribution $P(f_{u}| obj_i, pred_k)$ or from the pooled distribution $P(f_{u}| pred_k)$. The feature extractor is trying to confuse the discriminator:
\begin{equation}\label{eqn:cadvdajddiv_D}
\vspace{-1mm}
\begin{aligned}
    \min_{F} \max_{D} \sum_{k\in[K]} & \alpha^{(k)} \sum_{i\in [M]} \alpha_{i}^{(k)}  \big( E_{f \in P(f_{u}| obj_i, pred_k)}[\log D(f, obj_i; pred_k)] \\
    & + E_{f \in P(f_{u}| pred_k)}[\log (1-D(f, obj_i; pred_k))] \big).
\end{aligned}
\vspace{-1mm}
\end{equation}
We show that in the appendix, assuming infinite capacity of $D$, the maximum of \eqref{eqn:cadvdajddiv_D} (up to a constant) is indeed $CJSD$ defined in \eqref{eqn:cadvdajddiv}. We propose to use $\alpha^{(k)} = N^{(k)}/N$ and $\alpha_{i}^{(k)} = N_{i}^{(k)}/N^{(k)}$, as in CADG-KLD. In practice, we optimize this with SGD, in which for each sample $x$ (with label $pred(x)$ and domain $obj(x)$), we update the $D$ and $F$ with the following regularization:
\begin{equation}\label{eqn:cadvdajddiv_Derm_sgd}
\vspace{-1mm}
\begin{aligned}
    & \min_{F} \max_{D} ~ \log D(F(x), obj(x); pred(x)) +  \\
    & \sum_{i\in [M]} \frac{N_{i}^{pred(x)}}{N^{pred(x)}} \log (1-D(F(x), obj_i; pred(x))).
\end{aligned}
\vspace{-1mm}
\end{equation}



\subsubsection{A general recipe for ADG}
\label{subsec:generalrecipe}
Finally, we summarize a general recipe consisting of 3 steps for AGD training. First, one chooses the invariance to be enforced, such as the unconditional feature distribution matching~\eqref{eqn:unconditionDA} or the conditional matching~\eqref{eqn:cadvda}. Second, one chooses the distribution divergence (or distance) to be used as the DG regularization, such as KL divergence~\eqref{eqn:advdakldiv}\eqref{eqn:cadvdakldiv} and JD divergence~\eqref{eqn:cadvdajddiv}. With many GAN variants, one can freely pick many other divergences/distances, see, e.g., \cite{arjovsky2017wasserstein,li2017mmd,nowozin2016f,zhang17gan}. Third, one utilizes various GAN formulations, converts the DG regularization into an adversarial problem, and then performs the adversarial training with SGD, like in \eqref{eqn:unconditionDA_Dierm_sgd}, \eqref{eqn:cadvdakldiv_Derm_sgd} or \eqref{eqn:cadvdajddiv_Derm_sgd}.

\subsubsection{Relation between the proposed framework and the DeepC and CIDDG \cite{li2018deep}}
\label{subsec:ourvsdeepc}
The loss function of DeepC is a special case of our general framework: when $\alpha^{(k)} = \alpha_i^{(k)} = 1$ in \eqref{eqn:cadvdakldiv_Di}, our CADG-KLD reduces to DeepC. We propose to use $\alpha_{i}^{(k)} = N_{i}^{(k)}/N^{(k)}$, which gives significantly better results. Thank to our framework, the meaning of these parameters and our choice are very intuitive: domains with larger sample sizes should contribute more to the pooled distribution~\eqref{eqn:cadvdadist} and should have larger weights in the distribution matching regularization~\eqref{eqn:cadvdakldiv_Di}.
CIDDG reweights conditional distributions $P(f_{u}| obj_i, pred_k)$ with $1/\alpha_{i}^{(k)}$ to compute their class prior-normalized minimax value. This cannot be applied in our case, because many $\alpha_{i}^{(k)}$'s are 0 since many predicate-object pairs have never been seen in training set. After all, the large number of domains and huge variation across domains (i.e., $\alpha_{i}^{(k)}=0$ for many $(i,k)$'s) make the off-the-shelf DeepC have poor performance in the novel HOI detection task. Our results in Table \ref{tb:newsplit} and \ref{tb:unrelsplit} prove the advantage of our methods over DeepC.

\subsection{Architectures}\label{sec:architectures_ADG}
\noindent Fig.~\ref{fig:all_frameworks} shows the network structure for ADG-KLD, CADG-KLD, and CADG-JSD. As Fig.~\ref{fig:all_frameworks}(b) shows, the adversarial branch (blue path) of ADG-KLD and the union-box branch (red path) are trained iteratively in an adversarial manner. 
CADG-KLD is a conditional version of ADG-KLD, where the adversarial branch takes the predicate embedding as an additional input (Fig.~\ref{fig:all_frameworks}(d)). The goal of the adversarial branch in CADG-JSD is to distinguish whether the input feature is from the object-specific distribution or the pooled distribution, given the predicate. Therefore, as shown in Fig.~\ref{fig:all_frameworks}(e), it takes the feature, the object embedding and predicate embedding as inputs and predicts a binary output.

\section{Experiments}\label{sec:exp}


\noindent\textbf{Implementation details}
\noindent For feature extraction backbone, we adopt ResNet-50~\cite{he2016deep} and follow the setting of \cite{gao2018ican,li2018transferable}. There are three branches in our baseline model, and our domain generalization framework is only applied on the union-box branch.
It takes around 60 hours on 4 NVIDIA P100 GPU for training the model, and we we apply the linear scaling rule according to ~\cite{goyal2017accurate}. 


\subsection{Evaluations}\label{subsec:evluations}
\noindent\textbf{Baseline} To achieve a fair comparison, we first show that our strong baseline performs better than other approaches with similar architectures~\cite{chao2018learning,fgkioxari2018detecting,gao2018ican,qi2018learning} on the original split of HICO-DET dataset in Table~\ref{tb:originalsplit}. Then we use this baseline model in the new split for comparison with our proposed framework. In our experiments, we focus on examining whether our proposed adversarial training as a plug-in module can robustly improve the baseline. It can be easily applied to other recent HOI detection methods~\cite{li2019transferable,gupta2019no,wan2019pose,wang2019deep,zhou2019relation}.

\begin{table}[!ht]
\setlength{\tabcolsep}{2.5pt} 
\center
\resizebox{0.7\linewidth}{!}{
  \begin{tabular}{l c c c c c c }
   \toprule
  {} 
         & \multicolumn{3}{c}{Default}        
         & \multicolumn{3}{c}{Known Object} \\    
   \cmidrule(r){2-4}
   \cmidrule(r){5-7}
 Method & Full & Rare & Non Rare &   Full & Rare   & Non Rare  \\
    \cmidrule(r){1-7}
   \multirow{1}{*}{}
 HO-RCNN~\cite{chao2018learning}  & 7.81 & 5.37 & 8.54 &   10.41 & 8.94   & 10.85\\
    \cmidrule(r){1-7}
    \multirow{1}{*}{}
 InteractNet~\cite{fgkioxari2018detecting}  & 9.94 & 7.16 & 10.77 &   - & -   & -\\
 \cmidrule(r){1-7}
    \multirow{1}{*}{}
 iCAN~\cite{gao2018ican}  & 14.84 & 10.45 & 16.15 &   16.26 & 11.33   & 17.73 \\
 \cmidrule(r){1-7}
    \multirow{1}{*}{}
 GPNN~\cite{qi2018learning}  & 13.11 & 9.34 & 14.23 &   -- & --   & -- \\
 \cmidrule(r){1-7}
    \multirow{1}{*}{}
 Baseline  & \textbf{15.27} & \textbf{11.82} & \textbf{16.31} &   \textbf{16.97} & \textbf{13.94}   & \textbf{17.87}\\
\bottomrule
\end{tabular}
}
\vspace{2mm}
\caption{Performance on the original HICO-DET dataset}
\label{tb:originalsplit}
\vspace{-5mm}
\end{table}

\begin{table*}[t]
\footnotesize
\setlength{\tabcolsep}{2.5pt} 
\center
\resizebox{1.0\linewidth}{!}{
  \begin{tabular}{c l c c c c c c c c c c c c}
   \toprule
  & \multicolumn{1}{c}{}     
         & \multicolumn{4}{c}{trainval}        
         & \multicolumn{4}{c}{testval}     
         & \multicolumn{4}{c}{test}       \\
  \cmidrule(r){3-6}
  \cmidrule(r){7-10}
  \cmidrule(r){11-14}
  & \multicolumn{1}{c}{}     
         & \multicolumn{2}{c}{PredCls}        
         & \multicolumn{2}{c}{PredDet}     
         & \multicolumn{2}{c}{PredCls}   
         & \multicolumn{2}{c}{PredDet}  
         & \multicolumn{2}{c}{PredCls}  
         & \multicolumn{2}{c}{PredDet}    \\

   \cmidrule(r){3-4}
   \cmidrule(r){5-6}
   \cmidrule(r){7-8}
   \cmidrule(r){9-10}
   \cmidrule(r){11-12}
   \cmidrule(r){13-14}
 & Method  & R@1 & R@5 & R@5 & R@10  & R@1  & R@5  &  R@5  & R@10  &  R@1  & R@5 &  R@5  & R@10    \\
    \cmidrule(r){2-14}
   \multirow{3}{*}
 & Frequency & 43.04 & 95.25 & 54.23 & 69.58 & 0.00 & 1.93 & 0.25 & 2.49 & 0.00 & 0.15 & 0.12 & 2.43  \\
 & Baseline & 41.87 & 91.09 & 51.17 & 66.92 & 32.03 & 75.95 & 52.96 & 68.88 & 32.18 & 76.73 & 51.92 & 67.45  \\
 & DeepC~\cite{li2018deep} & 40.58 (\textbf{-3.1\%}) & 90.41 & 50.21 & 65.85 & 32.82 (\textbf{+2.5\%}) & 76.92 & 53.89 & 69.54 & 32.80 (\textbf{+1.9\%}) & 77.62 & 52.33 & 67.95  \\
 \cmidrule(r){2-14}
   \multirow{3}{*}
 & ADG-KLD & 40.08 (\textbf{-4.3\%}) & 89.85 & 49.72 & 65.94 & 46.78 (\textbf{+46.1\%}) & 80.27 & 61.05 & 74.48 & 48.68 (\textbf{+51.3\%}) & 81.61 & 60.98 & 74.34  \\
 & CADG-KLD & 39.92 (\textbf{-4.7\%}) & 88.09 & 49.07 & 64.66 & 40.33 (\textbf{+25.9\%}) & 75.85 & 55.04 & 69.43 & 41.66 (\textbf{+29.5\%}) & 76.89 & 54.88 & 68.96 \\
 & CADG-JSD & 40.15 (\textbf{-4.1\%}) & 88.48 & 50.12 & 65.38 & 42.29 (\textbf{+32.0\%}) & 76.60 & 56.10 & 69.68 & 43.47 (\textbf{+35.1\%}) & 77.55 & 56.08 & 69.26  \\
\bottomrule
\end{tabular}
}
\vspace{2mm}
\caption{Performance on the new split of HICO-DET dataset. For ~\cite{li2018deep}, ADG-KLD, CADG-KLD, and CADG-JSD, we measure the relative ratio with baseline to compute the gain and loss on PredCls R@1}
\label{tb:newsplit}
\vspace{-5mm}
\end{table*}

\noindent\textbf{HICO-DET Dataset} We conducted extensive experiments to hyper-tune the baseline method on the new split in order to achieve its best performance, as shown in Table~\ref{tb:newsplit}. We observe that the baseline models perform much worse in test set than in trainval set, which is expected and indicates that the baseline model has very limited generalization ability. Another simple baseline model is the \textbf{frequency} model, which gets the statistics of the <human, predicate, object> triplets and predicts the predicate class only based on the frequency. From Table~\ref{tb:newsplit} we observe that the frequency model can only make random predictions on the test set as expected, because the triplets in the test set do not exist in the training set. Moreover, in the trainval set, the performance of the frequency model is even a little higher with our baseline models, which indicates that the baseline models are mostly learning from the frequency bias in the training set. This is also observed from Visual Genome dataset~\cite{krishna2017visual} by Zellers \emph{et al.}~\cite{zellers2018neural}. Besides, DeepC~\cite{li2018deep} proposes a deep domain generalization approach which is only applied on some toy datasets. We extend this method on the new split as another baseline method.

We compare our proposed adversarial domain generalization framework with the baseline models. Our proposed methods all decrease by around 4\% on PredCls R@1 of the trainval set, which is reasonable as the proposed model disentangles object and predicate representations. On the other hand, while DeepC~\cite{li2018deep} does not show much improvement over the baseline, our models gain around 26\%\texttildelow 51\% on testval and test set.
As discussed in Section~\ref{subsec:ourvsdeepc}, DeepC is a special case of CADG-KLD with uniform weights across domains, while we propose to use natural weights. From Table~\ref{tb:newsplit} and \ref{tb:unrelsplit}, this change improve CADG-KLD's performance significantly.

\noindent\textbf{UnRel Dataset}
To further investigate the generalization ability of the proposed models on novel relation triplets, we evaluate the metrics on UnRel dataset directly using the models trained on HICO-DET which are the same as the models in Table~\ref{tb:newsplit}. The evaluation results are shown in Table~\ref{tb:unrelsplit}. While we observe minor performance gain on DeepC~\cite{li2018deep}, all our proposed models show significant improvements uniformly on the metrics, and ADG-KLD and CADG-JSD both have even better performance, with an increase of over 75\% and 125\% comparing with the baseline model. Note that UnRel dataset has triplets with unseen object classes and even non-human subjects, which indicates that our proposed models generalize better to unseen triplet categories than baseline.

\begin{table}[ht]
\footnotesize
\setlength{\tabcolsep}{2.5pt} 
\center
\resizebox{0.7\linewidth}{!}{
  \begin{tabular}{c l  c c c c c c}
   \toprule
  & \multicolumn{1}{c}{}         
         & \multicolumn{2}{c}{split 1}     
         & \multicolumn{2}{c}{split 2}   
         & \multicolumn{2}{c}{split 3}  \\
  \cmidrule(r){3-4}
  \cmidrule(r){5-6}
  \cmidrule(r){7-8}
 & Method  & R@1 & R@5 & R@1 & R@5 & R@1 & R@5   \\
    \cmidrule(r){2-8}
   \multirow{3}{*}
 & Frequency & 0.00 & 0.00 & -- & -- & -- & --   \\
 & Baseline & 13.31 & 63.47 & 17.68 & 66.41 & 14.76 & 61.00 \\
 & DeepC~\cite{li2018deep} & 16.72 & 65.94 & 21.97 & 67.93 & 14.90 (\textbf{+1\%}) & 63.51 \\
 \cmidrule(r){2-8}
   \multirow{3}{*}
 & ADG-KLD & 39.01 & 74.30 & 41.41 & 76.52 & 33.15 (\textbf{+125\%}) & 69.22 \\
 & CADG-KLD & 19.20 & 69.35 & 26.01 & 71.21 & 20.19 (\textbf{+37\%}) & 60.86  \\
 & CADG-JSD & 27.24 & 73.99 & 34.09 & 75.51 & 26.04 (\textbf{+75\%}) & 64.07  \\
\bottomrule
\end{tabular}
}
\vspace{2mm}
\caption{Performance of the evaluation metric PredCls on the UnRel dataset. For ~\cite{li2018deep}, ADG-KLD, CADG-KLD, and CADG-JSD, we measure the relative ratio with baseline to compute the gain and loss on R@1}
\label{tb:unrelsplit}
\vspace{-5mm}
\end{table}

\noindent\textbf{Qualitative Results}
We show our human-object interaction detection results in Figure~\ref{fig:qualitative}, where each subplot illustrates one <human, predicate, object> triplet. We choose the predicate with the top score of each instance for visualization and comparison. Images of the first three columns are from the HICO-DET dataset. 
We note that our models perform uniformly better than the baseline, and detect the predicates when facing unseen triplets in the test images. The last column of  Figure~\ref{fig:qualitative} shows rare scenes with unseen triplets in the images, and there is even one instance that takes a cat as the subject. This implies that our models  learn better features of predicates themselves which is invariant to objects so as to get strong generalization ability.

\begin{figure*}[ht]
\centering
\scriptsize
\setlength{\tabcolsep}{1pt}
\begin{tabular}{ccc}
 \includegraphics[width=.33\textwidth]{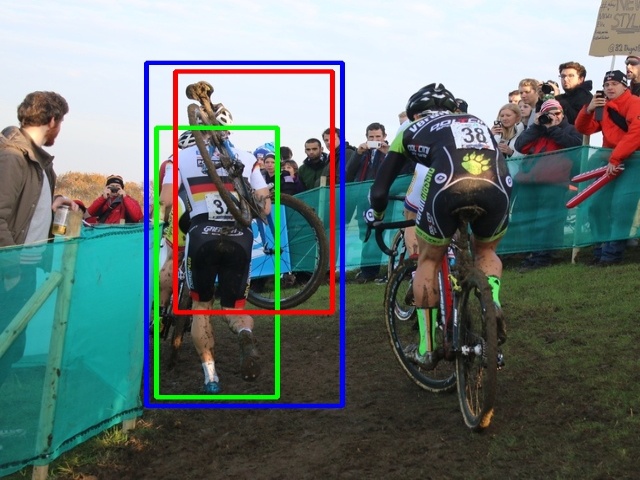}&
 \includegraphics[width=.33\textwidth]{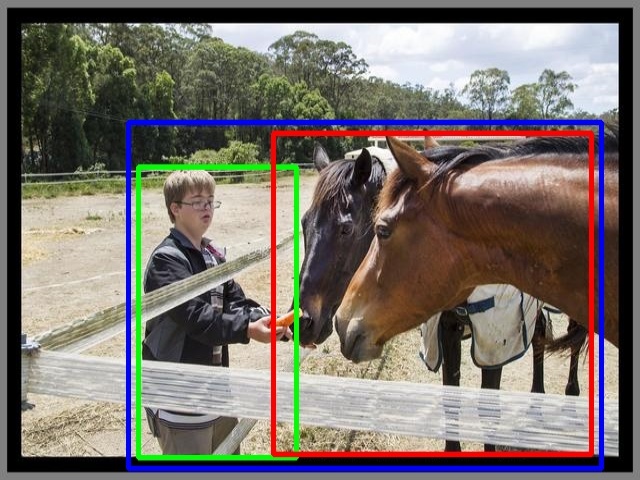}&
 \includegraphics[width=.33\textwidth]{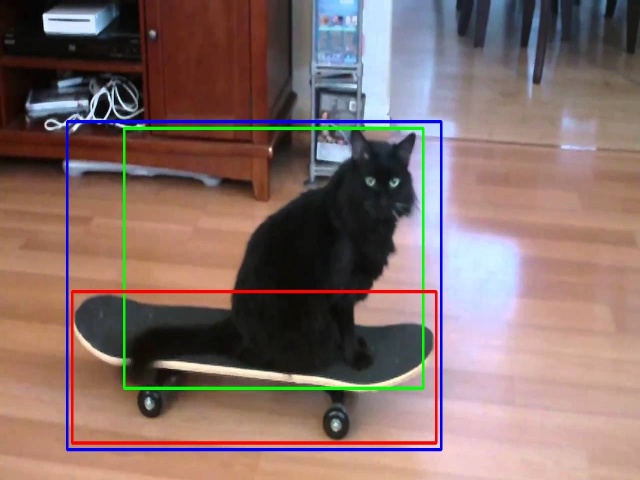}\\
 \tabincell{l}{
    \textbf{Baseline}: \textcolor{red}{straddle}\\
    \textbf{ADG-KLD}: \textcolor{darkgreen}{carry}\\
    \textbf{CADG-KLD}: \textcolor{red}{hold}\\
    \textbf{CADG-JSD}: \textcolor{darkgreen}{carry}} 
& \tabincell{l}{
    \textbf{Baseline}: \textcolor{red}{watch}\\        
    \textbf{ADG-KLD}: \textcolor{red}{no interaction}\\
    \textbf{CADG-KLD}: \textcolor{darkgreen}{feed}\\
    \textbf{CADG-JSD}: \textcolor{darkgreen}{feed}} 
& \tabincell{l}{
    \textbf{Baseline}: \textcolor{red}{jump}\\        
    \textbf{ADG-KLD}: \textcolor{darkgreen}{ride}\\
    \textbf{CADG-KLD}: \textcolor{darkgreen}{ride}\\
    \textbf{CADG-JSD}: \textcolor{darkgreen}{ride}}  \\
 \includegraphics[width=.33\textwidth]{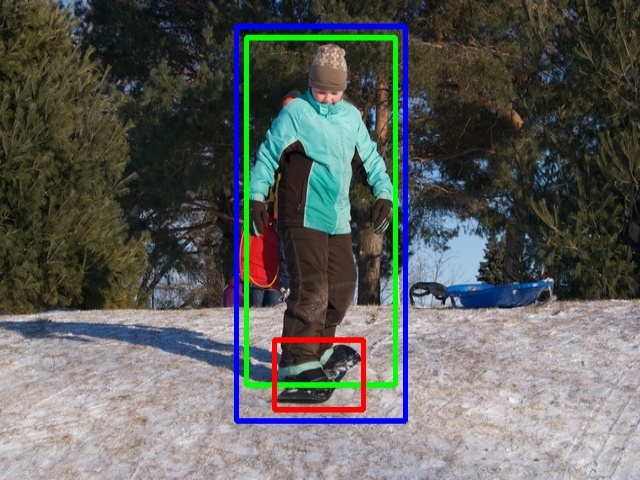}&
 \includegraphics[width=.33\textwidth]{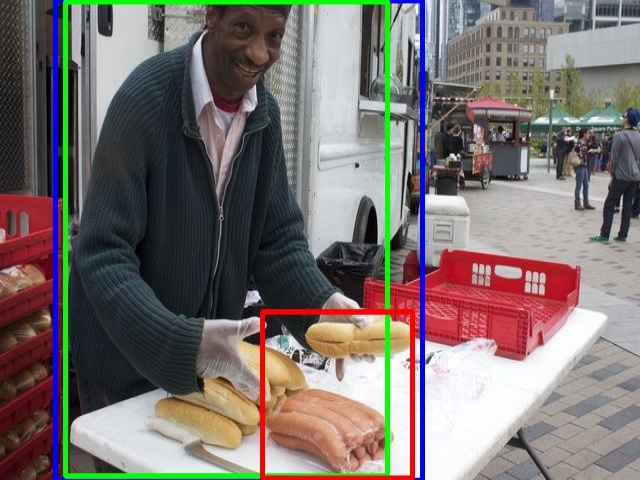}&
 \includegraphics[width=.33\textwidth]{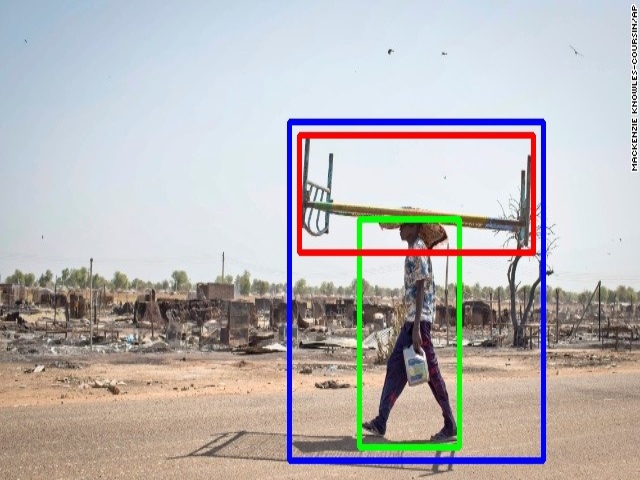}\\
 \tabincell{l}{
    \textbf{Baseline}: \textcolor{red}{wear}\\        
    \textbf{ADG-KLD}: \textcolor{darkgreen}{ride}\\
    \textbf{CADG-KLD}: \textcolor{darkgreen}{ride}\\
    \textbf{CADG-JSD}: \textcolor{darkgreen}{ride}} 
& \tabincell{l}{
    \textbf{Baseline}: \textcolor{red}{cut}\\        
    \textbf{ADG-KLD}: \textcolor{darkgreen}{hold}\\
    \textbf{CADG-KLD}: \textcolor{red}{cut}\\
    \textbf{CADG-JSD}: \textcolor{red}{cut}}  
& \tabincell{l}{
    \textbf{Baseline}: \textcolor{red}{fly}\\        
    \textbf{ADG-KLD}: \textcolor{darkgreen}{carry}\\
    \textbf{CADG-KLD}: \textcolor{red}{hold}\\
    \textbf{CADG-JSD}: \textcolor{darkgreen}{carry}} \\
\end{tabular}
\caption{Qualitative results on test images. \textbf{Green box:} human. \textbf{Red box:} object. \textbf{Blue box:} union box of object and human with a margin. \textbf{Green text} indicates correct predictions and \textbf{red text} implies the wrong ones. Images of first two columns are from the \textbf{HICO-DET} dataset, and images of the last column are from the \textbf{UnRel} dataset.}
\vspace{-12pt}
\label{fig:qualitative}
\end{figure*} 

\subsection{Analysis}\label{subsec:analysis}

\noindent\textbf{Per-class evaluations}
We evaluate the PredCls R@1 for each class in Table~\ref{tb:perclassrecall}. Among all the predicates, we choose some representative verbs including both high-frequency and low-frequency ones, where the frequency reflects the number of positive instances in training. For each predicate, we make the statistics of the number of instances in the test set and evaluate the R@1 metric. We observe from the table that our proposed models improve significantly over the baseline model for high-frequency predicates, and get comparable results for low-frequency predicates. 
We also evaluate the mean R@1 which is an average of the per-class R@1, where we still get an improvement of around 10\%\texttildelow 26\% over the baseline.

\begin{table}[ht]
\footnotesize
\setlength{\tabcolsep}{2.5pt} 
\center
\resizebox{0.9\linewidth}{!}{
  \begin{tabular}{c l  c c c c c}
   \toprule

  & \multicolumn{1}{c}{} 
         & \multicolumn{1}{c}{\# instances}         
         & \multicolumn{1}{c}{Baseline}  
         & \multicolumn{1}{c}{ADG-KLD}         
         & \multicolumn{1}{c}{CADG-KLD} 
         & \multicolumn{1}{c}{CADG-JSD}      \\
    \cmidrule(r){2-7}
   \multirow{4}{*}
 & hold  & 14956 & 43.50 & 61.64 & 70.19 & 72.92  \\
 & ride  & 13967  & 27.32  & 88.98  & 66.84 & 70.94  \\
 & sit on  & 11051 & 22.08 & 32.93  & 16.73  & 19.10  \\
 & carry  & 4526 & 1.82 & 12.87 & 2.48 & 3.96  \\
 \cmidrule(r){2-7}
   \multirow{6}{*}
  & watch  & 1553 & 7.89 & 11.84 & 14.47 & 15.79  \\
  & walk  & 673 & 27.27 & 1.30 & 22.08 & 35.06  \\
  & feed  & 555 & 6.15 & 1.54 & 6.15 & 7.69  \\
  & cut  & 367 & 5.71 & 8.57 & 11.43 & 8.57  \\
  & \tabincell{l}{push, exit, etc.}  & <250 & 0.00 & 0.00 & 0.00 & 0.00  \\
 \cmidrule(r){2-7}
   \multirow{2}{*}
  & overall  & -- & 32.18 & 48.68 (\textbf{+51.3\%}) & 41.66 (\textbf{+29.5\%}) & 43.47 (\textbf{+35.1\%}) \\
  & mean  & -- & 5.96 & 6.53 (\textbf{+9.6\%}) & 7.52 (\textbf{+26.2\%}) & 6.70 (\textbf{+12.4\%})  \\
\bottomrule
\end{tabular}
}
\vspace{2mm}
\caption{Per-class evaluations of PredCls R@1 on HICO-DET test set. The second column indicates the number of instances of each predicate in the training set. We show the baseline and our proposed models for each action}
\label{tb:perclassrecall}
\vspace{-5mm}
\end{table}

\begin{table}[ht]
\footnotesize
\setlength{\tabcolsep}{2.5pt} 
\center
\resizebox{.9\linewidth}{!}{
  \begin{tabular}{c l  c c c c c c}
   \toprule
  & \multicolumn{1}{c}{}         
         & \multicolumn{2}{c}{trainval}     
         & \multicolumn{2}{c}{testval}   
         & \multicolumn{2}{c}{test}     \\
  \cmidrule(r){3-4}
  \cmidrule(r){5-6}
  \cmidrule(r){7-8}
  & \multicolumn{1}{c}{} 
         & \multicolumn{1}{c}{PredCls}         
         & \multicolumn{1}{c}{PredDet}  
         & \multicolumn{1}{c}{PredCls}         
         & \multicolumn{1}{c}{PredDet} 
         & \multicolumn{1}{c}{PredCls}         
         & \multicolumn{1}{c}{PredDet}    \\
  \cmidrule(r){3-3}
  \cmidrule(r){4-4}
  \cmidrule(r){5-5}
  \cmidrule(r){6-6}
  \cmidrule(r){7-7}
  \cmidrule(r){8-8}
 & Method  & R@1 & R@5 & R@1 & R@5 & R@1 & R@5    \\
    \cmidrule(r){2-8}
   \multirow{2}{*}
 & Baseline(HSp)  & 37.25 & 48.05 & 28.31 & 51.04 & 28.35 & 50.68  \\
 & Baseline(full) & 41.87 & 51.17 & 32.03 & 52.96 & 32.18 (\textbf{+13.5\%}) & 51.92  \\
 \cmidrule(r){2-8}
   \multirow{3}{*}
 & ADG-KLD(HSp) & 41.34 & 51.09 & 43.74 & 59.68 & 44.94 & 59.34  \\
 & ADG-KLD(full) & 40.08 & 49.72 & 46.78 & 61.05 & 48.68 (\textbf{+8.3\%}) & 60.98  \\
 & CADG-KLD(HSp) & 36.74 & 47.55 & 25.57 & 49.06 & 26.33 & 48.24  \\
 & CADG-KLD(full) & 39.92 & 49.07 & 40.33 & 55.04 & 41.66 (\textbf{+58.2\%}) & 54.88  \\
 & CADG-JSD(HSp) & 37.65 & 48.74 & 26.13 & 49.93 & 22.25 & 49.30  \\
 & CADG-JSD(full) & 40.15 & 50.12 & 42.29 & 56.10 & 43.47 (\textbf{+95.4\%}) & 56.08  \\
\bottomrule
\end{tabular}
}
\vspace{2mm}
\caption{Ablation study on HICO-DET dataset. For each model, we show its inference results using the full model and HSp branches. The relative gain of the union-box branch of each model is also calculated}
\label{tb:unionboxbranch}
\vspace{-5mm}
\end{table}

\noindent\textbf{Ablation study}
In Table~\ref{tb:unionboxbranch}, we evaluate the contributions of the object branch in our full model to the results. HSp inference represents the inference using only the human-box branch (H) and spatial branch (Sp), while the full model makes use of all three branches for prediction. Comparing with the baseline model which gains 13.5\% in test set, ADG-KLD shows that this unconditional domain generalization method cannot get very good features on the union-box branch, and the adversarial training process indirectly optimizes the human-box branch and the spatial branch. On the other hand, the union-box branch of CADG-KLD and CADG-JSD shows significant improvement on the model, with around 60\%\texttildelow 100\% in the test set. This indicates that the union-box branch of these conditional domain generalization models learn much better features than the unconditional one, which is also the reason that we go beyond the unconditional model even if it shows overall great performance.

\begin{figure*}[ht]
\centering
\small
\setlength{\tabcolsep}{1pt}
\begin{tabular}{ccccc}
 \includegraphics[width=.2\textwidth]{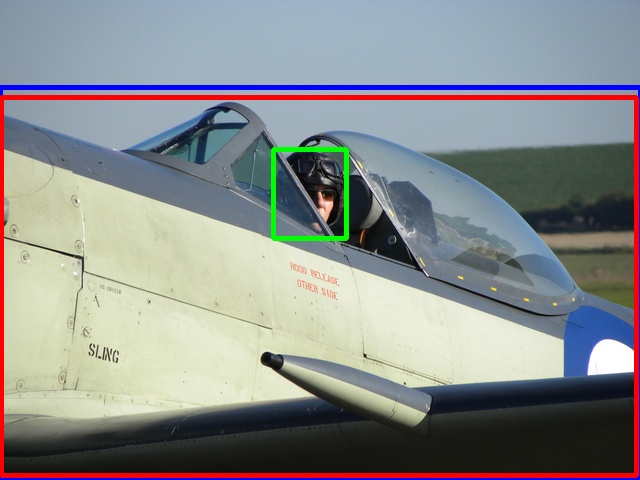}&
 \includegraphics[width=.2\textwidth]{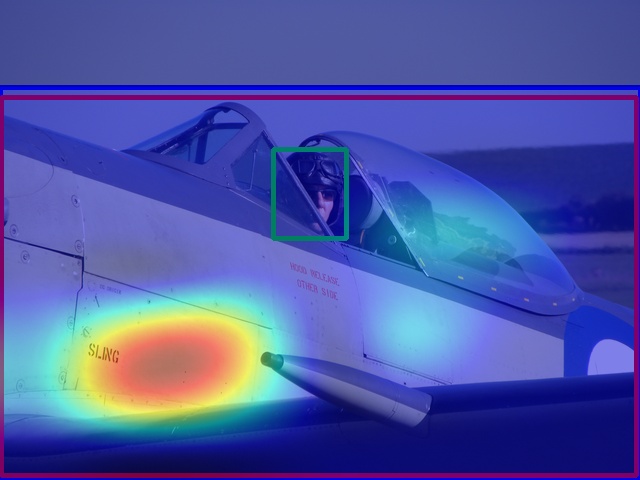}&
 \includegraphics[width=.2\textwidth]{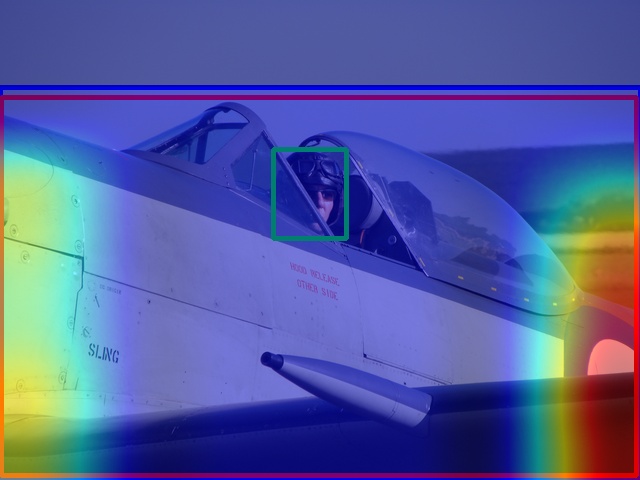}&
 \includegraphics[width=.2\textwidth]{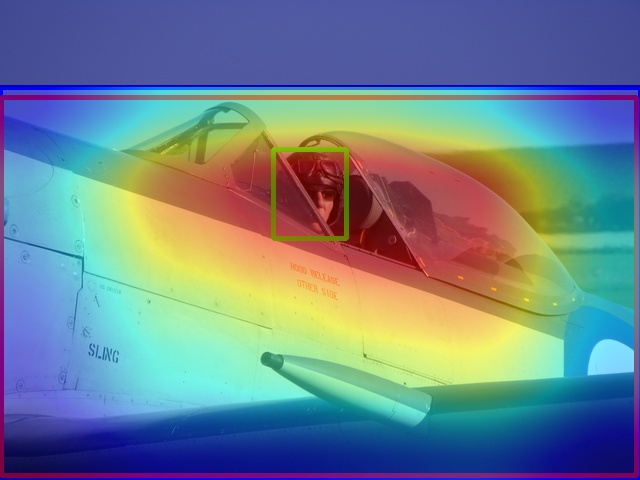}&
 \includegraphics[width=.2\textwidth]{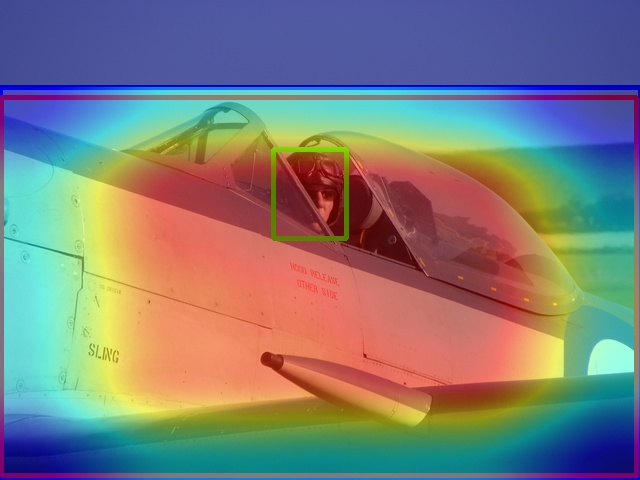}\\
 \includegraphics[width=.2\textwidth]{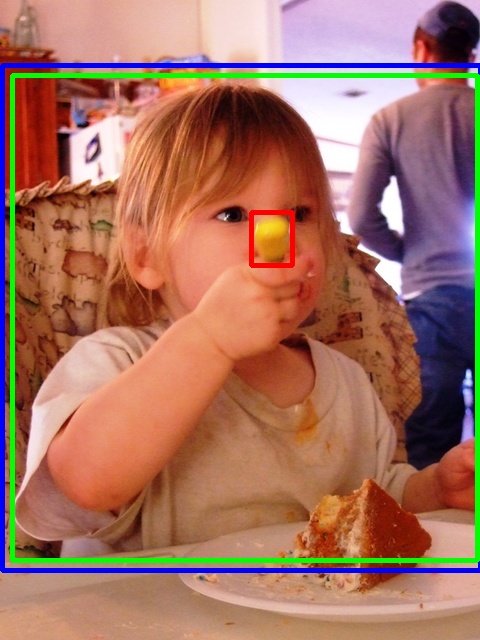}&
 \includegraphics[width=.2\textwidth]{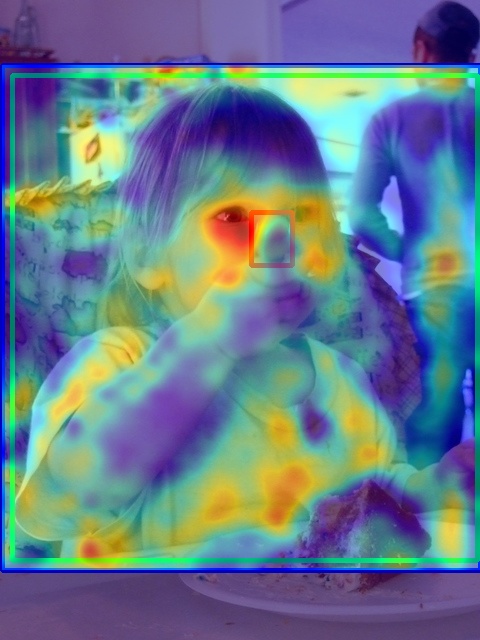}&
 \includegraphics[width=.2\textwidth]{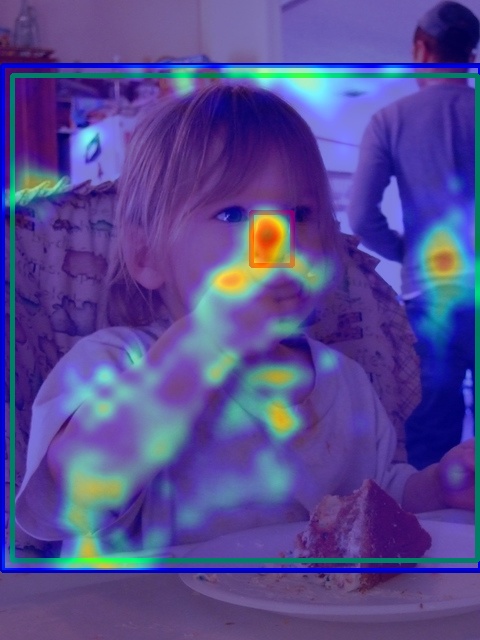}&
 \includegraphics[width=.2\textwidth]{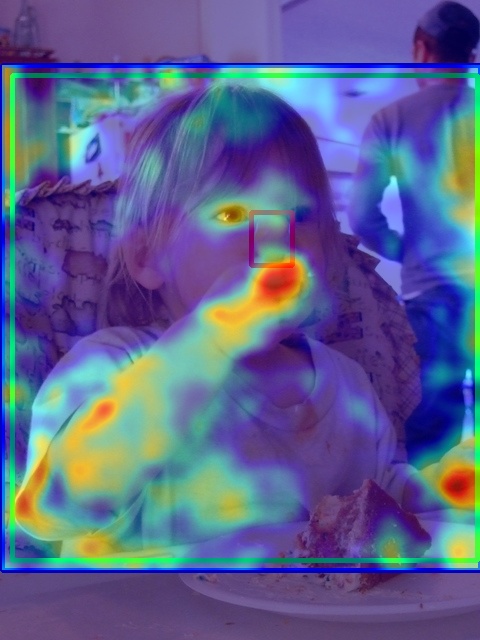}&
 \includegraphics[width=.2\textwidth]{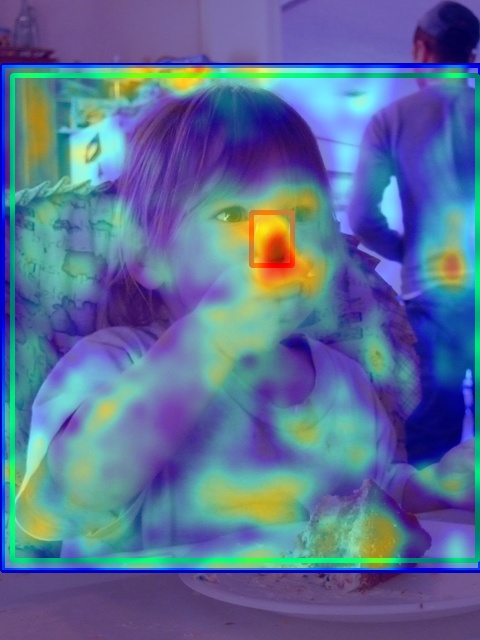}\\
(a) & (b) & (c) & (d) & (e) \\
\end{tabular}
\vspace{-8pt}
\caption{Grad-CAM visualization. \textbf{First row:} visualization of $ride$ from the \textbf{union-box} features. \textbf{Second row:} visualization of $hold$ from the \textbf{backbone} features before the ROI Align module (we only keep the visualization inside the union box). \textbf{Green box:} human. \textbf{Red box:} object. \textbf{Blue box:} union box of object and human with a margin. (a) input images. (b) baseline. (c) ADG-KLD. (d) CADG-KLD. (e) CADG-JSD. }
\vspace{-5pt}
\label{fig:gradcam}
\end{figure*} 

\noindent\textbf{Grad-CAM visualization}
Fig.~\ref{fig:gradcam} shows the visualization of intermediate features extracted from the input image, following Grad-CAM~\cite{selvaraju2017grad}, which takes a weighted average of the feature map w.r.t the gradient of the ground truth one-hot vector. The first row visualizes the feature map from the union-box branch and the second row visualizes the features before the ROI Align module. We note that the baseline features attends to the wrong region, while CADG-KLD and CADG-JSD pay more attention to the region that contains possible interaction between human and the object. The unconditional ADG-KLD gives a reverse saliency map comparing with CADG-KLD and CADG-JSD, which implies that ADG-KLD cannot learn good features on the union-box branch. This is the same as what we observe from the ablation study. On the other hand, as the second row shows, while the baseline feature focuses more on the eyes, all our models attend more to the interaction region between the hand and the folk, which demonstrates the effectiveness of the proposed domain generalization approaches.

To sum up, we show that our proposed ADG framework can get uniformly significant improvement over the baselines. Further analysis shows that, while the unconditional ADG indirectly optimize the human and spatial branch, the conditional ADG methods can improve the union-box features directly.

\vspace{-3mm}
\section{Conclusion}
\vspace{-2mm}
\noindent In this paper, we have focused on the problem of novel human-object interaction detection where the triplet combinations in the test set are unseen during training. To evaluate the performance in this setting, we created a new split based on the HICO-DET dataset and another evaluation set from the UnRel dataset. We proposed a unified adversarial domain generalization framework to tackle this problem. Experiments showed that our framework achieved significant improvement over the baseline models, by up to 50\% on the new split of HICO-DET test set and up to 125\% on the UnRel dataset. 
Our work shows that adversarial domain generalization is a promising way to overcome the combinatorial prediction problem in real-world applications.


\section*{Acknowledgement}
We thank Shuxi Zeng and Xiaodong Liu for valuable discussions and comments.

%
%
\bibliographystyle{splncs04}
\bibliography{reference.bib}

\begin{thebibliography}{10}
\providecommand{\url}[1]{\texttt{#1}}
\providecommand{\urlprefix}{URL }
\providecommand{\doi}[1]{https://doi.org/#1}

\bibitem{arjovsky2017wasserstein}
Arjovsky, M., Chintala, S., Bottou, L.: Wasserstein generative adversarial
  networks. In: International Conference on Machine Learning. pp. 214--223
  (2017)

\bibitem{chao2018learning}
Chao, Y.W., Liu, Y., Liu, X., Zeng, H., Deng, J.: Learning to detect
  human-object interactions. In: 2018 IEEE Winter Conference on Applications of
  Computer Vision (WACV). pp. 381--389. IEEE (2018)

\bibitem{chao2015hico}
Chao, Y.W., Wang, Z., He, Y., Wang, J., Deng, J.: Hico: A benchmark for
  recognizing human-object interactions in images. In: Proceedings of the IEEE
  International Conference on Computer Vision. pp. 1017--1025 (2015)

\bibitem{dai2017detecting}
Dai, B., Zhang, Y., Lin, D.: Detecting visual relationships with deep
  relational networks. In: Proceedings of the IEEE Conference on Computer
  Vision and Pattern Recognition. pp. 3076--3086 (2017)

\bibitem{ganin2014unsupervised}
Ganin, Y., Lempitsky, V.: Unsupervised domain adaptation by backpropagation.
  arXiv preprint arXiv:1409.7495  (2014)

\bibitem{gao2018ican}
Gao, C., Zou, Y., Huang, J.B.: ican: Instance-centric attention network for
  human-object interaction detection. arXiv preprint arXiv:1808.10437  (2018)

\bibitem{ghifary2016scatter}
Ghifary, M., Balduzzi, D., Kleijn, W.B., Zhang, M.: Scatter component analysis:
  A unified framework for domain adaptation and domain generalization. IEEE
  transactions on pattern analysis and machine intelligence  \textbf{39}(7),
  1414--1430 (2016)

\bibitem{ghifary2015domain}
Ghifary, M., Bastiaan~Kleijn, W., Zhang, M., Balduzzi, D.: Domain
  generalization for object recognition with multi-task autoencoders. In:
  Proceedings of the IEEE international conference on computer vision. pp.
  2551--2559 (2015)

\bibitem{fgkioxari2018detecting}
Gkioxari, G., Girshick, R., Doll{\'a}r, P., He, K.: Detecting and recognizing
  human-object interactions. In: Proceedings of the IEEE Conference on Computer
  Vision and Pattern Recognition. pp. 8359--8367 (2018)

\bibitem{goodfellow2014generative}
Goodfellow, I., Pouget-Abadie, J., Mirza, M., Xu, B., Warde-Farley, D., Ozair,
  S., Courville, A., Bengio, Y.: Generative adversarial nets. In: Advances in
  neural information processing systems. pp. 2672--2680 (2014)

\bibitem{goyal2017accurate}
Goyal, P., Doll{\'a}r, P., Girshick, R., Noordhuis, P., Wesolowski, L., Kyrola,
  A., Tulloch, A., Jia, Y., He, K.: Accurate, large minibatch sgd: Training
  imagenet in 1 hour. arXiv preprint arXiv:1706.02677  (2017)

\bibitem{gupta2015visual}
Gupta, S., Malik, J.: Visual semantic role labeling. arXiv preprint
  arXiv:1505.04474  (2015)

\bibitem{gupta2019no}
Gupta, T., Schwing, A., Hoiem, D.: No-frills human-object interaction
  detection: Factorization, layout encodings, and training techniques. In:
  Proceedings of the IEEE International Conference on Computer Vision. pp.
  9677--9685 (2019)

\bibitem{he2016deep}
He, K., Zhang, X., Ren, S., Sun, J.: Deep residual learning for image
  recognition. In: Proceedings of the IEEE conference on computer vision and
  pattern recognition. pp. 770--778 (2016)

\bibitem{khosla2012undoing}
Khosla, A., Zhou, T., Malisiewicz, T., Efros, A.A., Torralba, A.: Undoing the
  damage of dataset bias. In: European Conference on Computer Vision. pp.
  158--171. Springer (2012)

\bibitem{krishna2017visual}
Krishna, R., Zhu, Y., Groth, O., Johnson, J., Hata, K., Kravitz, J., Chen, S.,
  Kalantidis, Y., Li, L.J., Shamma, D.A., et~al.: Visual genome: Connecting
  language and vision using crowdsourced dense image annotations. International
  Journal of Computer Vision  \textbf{123}(1),  32--73 (2017)

\bibitem{li2017mmd}
Li, C.L., Chang, W.C., Cheng, Y., Yang, Y., P{\'o}czos, B.: Mmd gan: Towards
  deeper understanding of moment matching network. arXiv preprint
  arXiv:1705.08584  (2017)

\bibitem{li2018deep}
Li, Y., Tian, X., Gong, M., Liu, Y., Liu, T., Zhang, K., Tao, D.: Deep domain
  generalization via conditional invariant adversarial networks. In:
  Proceedings of the European Conference on Computer Vision (ECCV). pp.
  624--639 (2018)

\bibitem{li2017scene}
Li, Y., Ouyang, W., Zhou, B., Wang, K., Wang, X.: Scene graph generation from
  objects, phrases and region captions. In: Proceedings of the IEEE
  International Conference on Computer Vision. pp. 1261--1270 (2017)

\bibitem{li2018transferable}
Li, Y.L., Zhou, S., Huang, X., Xu, L., Ma, Z., Fang, H.S., Wang, Y.F., Lu, C.:
  Transferable interactiveness prior for human-object interaction detection.
  arXiv preprint arXiv:1811.08264  (2018)

\bibitem{li2019transferable}
Li, Y.L., Zhou, S., Huang, X., Xu, L., Ma, Z., Fang, H.S., Wang, Y., Lu, C.:
  Transferable interactiveness knowledge for human-object interaction
  detection. In: Proceedings of the IEEE Conference on Computer Vision and
  Pattern Recognition. pp. 3585--3594 (2019)

\bibitem{liang2018visual}
Liang, K., Guo, Y., Chang, H., Chen, X.: Visual relationship detection with
  deep structural ranking. In: Thirty-Second AAAI Conference on Artificial
  Intelligence (2018)

\bibitem{liang2017deep}
Liang, X., Lee, L., Xing, E.P.: Deep variation-structured reinforcement
  learning for visual relationship and attribute detection. In: Proceedings of
  the IEEE conference on computer vision and pattern recognition. pp. 848--857
  (2017)

\bibitem{lu2016visual}
Lu, C., Krishna, R., Bernstein, M., Fei-Fei, L.: Visual relationship detection
  with language priors. In: European Conference on Computer Vision. pp.
  852--869. Springer (2016)

\bibitem{muandet2013domain}
Muandet, K., Balduzzi, D., Sch{\"o}lkopf, B.: Domain generalization via
  invariant feature representation. In: International Conference on Machine
  Learning. pp. 10--18 (2013)

\bibitem{nowozin2016f}
Nowozin, S., Cseke, B., Tomioka, R.: f-gan: Training generative neural samplers
  using variational divergence minimization. In: Advances in Neural Information
  Processing Systems. pp. 271--279 (2016)

\bibitem{peyre2018detecting}
Peyre, J., Laptev, I., Schmid, C., Sivic, J.: Detecting unseen visual relations
  using analogies. arXiv preprint arXiv:1812.05736  (2018)

\bibitem{peyre2017weakly}
Peyre, J., Sivic, J., Laptev, I., Schmid, C.: Weakly-supervised learning of
  visual relations. In: Proceedings of the IEEE International Conference on
  Computer Vision. pp. 5179--5188 (2017)

\bibitem{qi2018learning}
Qi, S., Wang, W., Jia, B., Shen, J., Zhu, S.C.: Learning human-object
  interactions by graph parsing neural networks. In: Proceedings of the
  European Conference on Computer Vision (ECCV). pp. 401--417 (2018)

\bibitem{sadeghi2011recognition}
Sadeghi, M.A., Farhadi, A.: Recognition using visual phrases. In: CVPR 2011.
  pp. 1745--1752. IEEE (2011)

\bibitem{selvaraju2017grad}
Selvaraju, R.R., Cogswell, M., Das, A., Vedantam, R., Parikh, D., Batra, D.:
  Grad-cam: Visual explanations from deep networks via gradient-based
  localization. In: Proceedings of the IEEE International Conference on
  Computer Vision. pp. 618--626 (2017)

\bibitem{shen2018scaling}
Shen, L., Yeung, S., Hoffman, J., Mori, G., Fei-Fei, L.: Scaling human-object
  interaction recognition through zero-shot learning. In: 2018 IEEE Winter
  Conference on Applications of Computer Vision (WACV). pp. 1568--1576. IEEE
  (2018)

\bibitem{tang2020unbiased}
Tang, K., Niu, Y., Huang, J., Shi, J., Zhang, H.: Unbiased scene graph
  generation from biased training. arXiv preprint arXiv:2002.11949  (2020)

\bibitem{tzeng2015simultaneous}
Tzeng, E., Hoffman, J., Darrell, T., Saenko, K.: Simultaneous deep transfer
  across domains and tasks. In: Proceedings of the IEEE International
  Conference on Computer Vision. pp. 4068--4076 (2015)

\bibitem{tzeng2017adversarial}
Tzeng, E., Hoffman, J., Saenko, K., Darrell, T.: Adversarial discriminative
  domain adaptation. In: Proceedings of the IEEE Conference on Computer Vision
  and Pattern Recognition. pp. 7167--7176 (2017)

\bibitem{wan2019pose}
Wan, B., Zhou, D., Liu, Y., Li, R., He, X.: Pose-aware multi-level feature
  network for human object interaction detection. In: Proceedings of the IEEE
  International Conference on Computer Vision. pp. 9469--9478 (2019)

\bibitem{wang2019deep}
Wang, T., Anwer, R.M., Khan, M.H., Khan, F.S., Pang, Y., Shao, L., Laaksonen,
  J.: Deep contextual attention for human-object interaction detection. In:
  Proceedings of the IEEE International Conference on Computer Vision. pp.
  5694--5702 (2019)

\bibitem{xu2019learning}
Xu, B., Wong, Y., Li, J., Zhao, Q., Kankanhalli, M.S.: Learning to detect
  human-object interactions with knowledge. In: Proceedings of the IEEE
  Conference on Computer Vision and Pattern Recognition (2019)

\bibitem{xu2017scene}
Xu, D., Zhu, Y., Choy, C.B., Fei-Fei, L.: Scene graph generation by iterative
  message passing. In: Proceedings of the IEEE Conference on Computer Vision
  and Pattern Recognition. pp. 5410--5419 (2017)

\bibitem{xu2014exploiting}
Xu, Z., Li, W., Niu, L., Xu, D.: Exploiting low-rank structure from latent
  domains for domain generalization. In: European Conference on Computer
  Vision. pp. 628--643. Springer (2014)

\bibitem{yang2018graph}
Yang, J., Lu, J., Lee, S., Batra, D., Parikh, D.: Graph r-cnn for scene graph
  generation. In: Proceedings of the European Conference on Computer Vision
  (ECCV). pp. 670--685 (2018)

\bibitem{zellers2018neural}
Zellers, R., Yatskar, M., Thomson, S., Choi, Y.: Neural motifs: Scene graph
  parsing with global context. In: Proceedings of the IEEE Conference on
  Computer Vision and Pattern Recognition. pp. 5831--5840 (2018)

\bibitem{zhang2017visual}
Zhang, H., Kyaw, Z., Chang, S.F., Chua, T.S.: Visual translation embedding
  network for visual relation detection. In: Proceedings of the IEEE conference
  on computer vision and pattern recognition. pp. 5532--5540 (2017)

\bibitem{zhang2019graphical}
Zhang, J., Shih, K.J., Elgammal, A., Tao, A., Catanzaro, B.: Graphical
  contrastive losses for scene graph generation. arXiv preprint
  arXiv:1903.02728  (2019)

\bibitem{zhang17gan}
Zhang, P., Liu, Q., Zhou, D., Xu, T., He, X.: On the
  discrimination-generalization tradeoff in gans. CoRR  \textbf{abs/1711.02771}
  (2017), \url{http://arxiv.org/abs/1711.02771}

\bibitem{zhou2019relation}
Zhou, P., Chi, M.: Relation parsing neural network for human-object interaction
  detection. In: Proceedings of the IEEE International Conference on Computer
  Vision. pp. 843--851 (2019)

\end{thebibliography}

\appendix
\newpage
\section{Creating the new split of the HICO-DET dataset}
\label{app:suppnewsplitcreation}
We first combine the original training and test sets of the HICO-DET dataset to merge all the annotations together. As the original images in the HICO-DET are annotated at the HOI level, we merge the same pair of annotated boxes with IoU $\geq$ 0.7. We then follow two main principles to split the full dataset: the training set and test set should not have overlapping <human, predicate, object> triplet categories; the training set and test set should not have overlapping images. For each specific predicate, we split the objects related to this predicate into 90\% and 10\% for training set and test set. 

After we get the training and test sets, we further create the validation sets from the training set. We first divide the training set into 8/9 and 1/9 for ``train+trainval'' and ``testval'' (for novel HOI validation), following the same procedure of creating the training and test sets. Then we split {\it i.i.d.} the ``train+trainval'' set into 7/8 and 1/8 for the ``train'' and ``trainval'' sets, where the ``trainval'' set is for in-domain HOI validation (HOIs seen in the train set). We do the model selection based on the ``trainval''+``testval'' set, which gives us a good trade-off between in-domain and out-of-domain generalization. 


\section{Evaluation metrics}
\label{app:evaluation_metrics}
The most commonly used evaluation metric in HOI detection is the role mean average precision (role mAP)~\cite{gupta2015visual}, which takes the image as input and predicts each instance in the format of <human box, object box, HOI class>. If the correct HOI class is predicted, and both the human and the object bounding boxes meet the condition of IoUs $\geq$ 0.5 w.r.t the ground truth annotations, the instance is marked as a true positive. However, we don't use this metric for two reasons. First, this metric evaluates two steps, object detection and predicate detection, but we want to focus on the predicate prediction problem. Second, the HOI class is defined for the combination of predicate and object labels, which cannot be applied to novel HOI detection. We would like to disentangle the HOI class into predicate class and object class in order to generalize the metrics to unseen triplet categories.

In HICO-DET dataset, the only evaluation metric is the role mAP, but there are also other evaluation metrics that are commonly used in visual relationship detection and scene graph generation. Scene graph generation mostly uses predicate classification, scene graph classification, and scene graph generation as its metrics, and visual relationship detection uses predicate detection, phrase detection, and relation detection. Because of the incomplete annotations in most of the datasets on the relationships between different objects, we only use recall instead of precision for evaluation, such as $R@20$, $R@50$, and $R@100$. There is a fundamental difference between these two tasks: graph constraints. In scene graph generation, each edge of object pairs could only have one relation prediction, while in visual relationship detection, there is no graph constraints, which means that each edge could have multiple relation predictions. In our setting of HOI detection, there could be multiple predicate labels between the same human-object pair, so it is more suitable to use evaluation metrics in the visual relationship detection setting. Other differences are listed as follows:

In scene graph generation, the metrics are defined as:
\begin{itemize}[leftmargin=*]
    \item \textbf{Predicate Classification}: Given images, ground truth boxes, and object labels, we need to predict the predicate labels. The model should be able to judge whether two objects have relationships or not and make the predictions with graph constraints.

    \item \textbf{Scene Graph Classification}: Given images and ground truth boxes, we need to predict the object labels and the predicate labels. Usually the models need to predict the object models and the predicate labels simultaneously.

    \item \textbf{Scene Graph Generation}: Given only the images, we need to predict the ground truth boxes, the object labels, and the predicate labels with the requirement of graph constraints.
\end{itemize}

In visual relationship detection, the metrics are defined as:
\begin{itemize}[leftmargin=*]
    \item \textbf{Predicate Detection}: Given images, ground truth boxes, object labels, and object pairs that have relationships in the annotations, we need to predict the predicate labels without the graph constraints.

    \item \textbf{Phrase Detection}: Given only the images, we need to predict the union boxes of two interacting objects and their predicate labels. This is designed in early papers and not used much recently.

    \item \textbf{Relation Detection}: Given only the images, we need to predict the ground truth boxes, the object labels, and the predicate labels without the graph constraints.
\end{itemize}

As we stated in Section~\ref{sec:problemformulation}, the problem of relationship detection usually contains two steps: object detection and predicate detection. In this paper, we mainly focus on predicate prediction problem to learn the object-invariant features and thus we provide the ground truth boxes as input to the model and attend to the generalization ability in the novel HOI detection problem.

Under this circumstances where we are given the ground truth boxes, as our task is more related to visual relationship detection which has no graph constraints, we use \textbf{predicate detection} (PredDet) as one of our metrics. In visual relationship detection, people usually use $R@50$ and $R@100$ for evaluation. But different from the Visual Genome dataset or some other datasets, there are on average 3.1 relationships and relatively few object instances in the HICO-DET dataset. Therefore, we use $R@5$, $R@10$ for evaluation. We also introduce another metric which is inspired by the classification task and the image retrieval task. For each given human-object pair, we predict its predicate label based on the ranking result of all predicates. This is defined as \textbf{instance predicate detection} (PredCls), and we use $R@1$ and $R@5$ for evaluation.

For the UnRel dataset, in most cases each image only has one relationship pair, where using PredDet doesn't make much sense comparing with PredCls. Therefore, we only use PredCls for evalation on UnRel dataset which also uses $R@1$ and $R@5$.

\section{More experiment results}\label{app:moreexperiments}

\noindent\textbf{Model selection} We select the models and make hyper-parameter tuning on the validation set, as described in Section~\ref{app:suppnewsplitcreation}. The experiment results are listed in Table~\ref{tb:newsplitval}. While DeepC drops 1.2\% on the validation set, our proposed methods could get improvements of around 10\%, and achieve more significant performance on the test set, as illustrated in the main paper.

\begin{table*}[ht]
\footnotesize
\setlength{\tabcolsep}{2.5pt} 
\center
\resizebox{1.\linewidth}{!}{
  \begin{tabular}{c l c c c c c c c c c c c c c}
   \toprule
  & \multicolumn{1}{c}{}     
         & \multicolumn{4}{c}{trainval}        
         & \multicolumn{4}{c}{testval}     
         & \multicolumn{4}{c}{test}    
         & \multicolumn{1}{c}{val} \\
  \cmidrule(r){3-6}
  \cmidrule(r){7-10}
  \cmidrule(r){11-14}
  \cmidrule(r){15-15}
  & \multicolumn{1}{c}{}     
         & \multicolumn{2}{c}{PredCls}        
         & \multicolumn{2}{c}{PredDet}     
         & \multicolumn{2}{c}{PredCls}   
         & \multicolumn{2}{c}{PredDet}  
         & \multicolumn{2}{c}{PredCls}  
         & \multicolumn{2}{c}{PredDet}  
         & \multicolumn{1}{c}{PredCls}    \\

   \cmidrule(r){3-4}
   \cmidrule(r){5-6}
   \cmidrule(r){7-8}
   \cmidrule(r){9-10}
   \cmidrule(r){11-12}
   \cmidrule(r){13-14}
   \cmidrule(r){15-15}
 & Method  & R@1 & R@5 & R@5 & R@10  & R@1  & R@5  &  R@5  & R@10  &  R@1  & R@5 &  R@5  & R@10 & R@1   \\
    \cmidrule(r){2-15}
   \multirow{3}{*}
 & Frequency & 43.04 & 95.25 & 54.23 & 69.58 & 0.00 & 1.93 & 0.25 & 2.49 & 0.00 & 0.15 & 0.12 & 2.43 & 25.51  \\
 & Baseline & 41.87 & 91.09 & 51.17 & 66.92 & 32.03 & 75.95 & 52.96 & 68.88 & 32.18 & 76.73 & 51.92 & 67.45 & 37.86 \\
 & DeepC~\cite{li2018deep} & 40.58 (\textbf{-3.1\%}) & 90.41 & 50.21 & 65.85 & 32.82 (\textbf{+2.5\%}) & 76.92 & 53.89 & 69.54 & 32.80 (\textbf{+1.9\%}) & 77.62 & 52.33 & 67.95 & 37.42 (\textbf{-1.2\%}) \\
 \cmidrule(r){2-15}
   \multirow{3}{*}
 & ADG-KLD & 40.08 (\textbf{-4.3\%}) & 89.85 & 49.72 & 65.94 & 46.78 (\textbf{+46.1\%}) & 80.27 & 61.05 & 74.48 & 48.68 (\textbf{+51.3\%}) & 81.61 & 60.98 & 74.34 & 42.81 (\textbf{+13.1\%}) \\
 & CADG-KLD & 39.92 (\textbf{-4.7\%}) & 88.09 & 49.07 & 64.66 & 40.33 (\textbf{+25.9\%}) & 75.85 & 55.04 & 69.43 & 41.66 (\textbf{+29.5\%}) & 76.89 & 54.88 & 68.96 & 40.09 (\textbf{+6.9\%}) \\
 & CADG-JSD & 40.15 (\textbf{-4.1\%}) & 88.48 & 50.12 & 65.38 & 42.29 (\textbf{+32.0\%}) & 76.60 & 56.10 & 69.68 & 43.47 (\textbf{+35.1\%}) & 77.55 & 56.08 & 69.26 & 41.03 (\textbf{+8.4\%})  \\
\bottomrule
\end{tabular}
}
\vspace{0mm}
\caption{Performance on the new split of HICO-DET dataset. Comparing with Table~\ref{tb:newsplit}, we add the last column of the metric on val which shows our model selection criterion.}
\label{tb:newsplitval}
\vspace{-3mm}
\end{table*}

\noindent\textbf{Experiments on different insertion positions of the adversarial branch}
In the main paper, we introduce our proposed framework where the adversarial branch is inserted before the last classification layer of the union-box branch. On the other hand, there may be multiple positions to add the adversarial branch, and we show another possible position for adversarial branch. By this design, we add the adversarial branch to the position which is directly before two FC layers and one classification layer. Meanwhile, the network structure of the adversarial branch is symmetric with the main branch, i.e. including two FC layers and one classification layer. From the experiment results shown in Table~\ref{tb:newsplithico_p1s1}, we note that our proposed models ADG-KLD and CADG-JSD only get comparable performance as DeepC~\cite{li2018deep}. Although CADG-KLD has better performance on the test set, it drops 11.0\% on the trainval set, which is a big sacrifice on the common triplets and doesn't get a good tradeoff between common triplets and novel triplets. Therefore, it shows that the proposed domain generalization models can get better training when inserting the adversarial branch before the last classification layer. 

\begin{table*}[ht]
\footnotesize
\setlength{\tabcolsep}{2.5pt} 
\center
\resizebox{1.\linewidth}{!}{
  \begin{tabular}{c l c c c c c c c c c c c c c}
   \toprule
  & \multicolumn{1}{c}{}     
         & \multicolumn{4}{c}{trainval}        
         & \multicolumn{4}{c}{testval}     
         & \multicolumn{4}{c}{test}    
         & \multicolumn{1}{c}{val} \\
  \cmidrule(r){3-6}
  \cmidrule(r){7-10}
  \cmidrule(r){11-14}
  \cmidrule(r){15-15}
  & \multicolumn{1}{c}{}     
         & \multicolumn{2}{c}{PredCls}        
         & \multicolumn{2}{c}{PredDet}     
         & \multicolumn{2}{c}{PredCls}   
         & \multicolumn{2}{c}{PredDet}  
         & \multicolumn{2}{c}{PredCls}  
         & \multicolumn{2}{c}{PredDet}  
         & \multicolumn{1}{c}{PredCls}    \\

   \cmidrule(r){3-4}
   \cmidrule(r){5-6}
   \cmidrule(r){7-8}
   \cmidrule(r){9-10}
   \cmidrule(r){11-12}
   \cmidrule(r){13-14}
   \cmidrule(r){15-15}
 & Method  & R@1 & R@5 & R@5 & R@10  & R@1  & R@5  &  R@5  & R@10  &  R@1  & R@5 &  R@5  & R@10 & R@1   \\
    \cmidrule(r){2-15}
   \multirow{3}{*}
 & Frequency & 43.04 & 95.25 & 54.23 & 69.58 & 0.00 & 1.93 & 0.25 & 2.49 & 0.00 & 0.15 & 0.12 & 2.43 & 25.51  \\
 & Baseline & 41.87 & 91.09 & 51.17 & 66.92 & 32.03 & 75.95 & 52.96 & 68.88 & 32.18 & 76.73 & 51.92 & 67.45 & 37.86 \\
 & DeepC~\cite{li2018deep} & 40.58 (\textbf{-3.1\%}) & 90.41 & 50.21 & 65.85 & 32.82 (\textbf{+2.5\%}) & 76.92 & 53.89 & 69.54 & 32.80 (\textbf{+1.9\%}) & 77.62 & 52.33 & 67.95 & 37.42 (\textbf{-1.2\%}) \\
 \cmidrule(r){2-15}
   \multirow{3}{*}
 & ADG-KLD & 40.39 (\textbf{-3.5\%}) & 89.43 & 50.02 & 65.81 & 32.18 (\textbf{+0.5\%}) & 75.47 & 52.42 & 68.51 & 32.60 (\textbf{+1.3\%}) & 76.14 & 51.87 & 66.86 & 37.05 (\textbf{-2.1\%}) \\
 & CADG-KLD & 37.25 (\textbf{-11.0\%}) & 86.44 & 46.53 & 62.52 & 37.62 (\textbf{+17.5\%}) & 77.49 & 55.04 & 71.89 & 37.93 (\textbf{+17.9\%}) & 78.73 & 54.49 & 70.34 & 37.40 (\textbf{-1.2\%}) \\
 & CADG-JSD & 41.34 (\textbf{-1.1\%}) & 90.05 & 50.37 & 66.10 & 32.57 (\textbf{+1.7\%}) & 76.73 & 53.65 & 69.31 & 33.45 (\textbf{+3.9\%}) & 77.62 & 53.43 & 68.29 & 37.77 (\textbf{-0.2\%})  \\
\bottomrule
\end{tabular}
}
\vspace{0mm}
\caption{Performance of the early insertion of the adversarial branch on HICO-DET dataset.}
\label{tb:newsplithico_p1s1}
\vspace{-3mm}
\end{table*}

\section{Proofs and additional derivations}
\label{app:proofs}
In this section, we provide proofs for our algorithms, showing that they are indeed minimizing the divergences we claimed in Section~\ref{subsec:uncondADG} and \ref{sec:cadvda}.We also provide additional derivations on how we simplify the populational loss to the corresponding minibatch loss used in Stochastic Gradient Descent (SGD).

\begin{thm}\label{thm:adg_kld}
Define the pooled feature distribution $P(f_{u})$ as
\begin{equation}\label{eqn:advdadist_app}
	P(f_{u}) = \sum_{i\in [M]} \alpha_{i} P(f_{u}| obj_i), \quad \sum_{i\in [M]} \alpha_{i} = 1
\end{equation} 
and $\alpha_{i}$ is the relative importance of each domain. We introduce a discriminator $D: f_{u} \to [0,1]^{M}$ that tries to classify the domain (object category) based on the union-box feature. Mathematically, the discriminator is trying to maximize the likelihood of the ground-truth domain: 
\begin{equation}\label{eqn:unconditionDA_Di_app}
    \max_{D} \sum_{i\in [M]} \alpha_{i} E_{f \in P(f_{u}| obj_i)}[\log D_i(f)].
\end{equation}
Assume infinite capacity of the discriminator $D$, the maximum of \eqref{eqn:unconditionDA_Di_app} is a weighted summation of KL divergence between distributions (up to a constant):
\begin{equation}\label{eqn:advdakldiv_app}
    KLD := \sum_{i\in [M]} \alpha_{i} KL(P(f_{u}| obj_i) || P(f_{u})).
\end{equation} 
\end{thm}

\begin{proof}
Let $D^*$ be an optimal discriminator, i.e.,
\begin{equation*}
\begin{aligned}
    D^* &= \argmax_{D} \sum_{i\in [M]} \alpha_{i} E_{f \in P(f_{u}| obj_i)}[\log D_i(f)] \\
    &= \argmax_{D} \sum_{i\in [M]} \alpha_{i} \int_f P(f| obj_i) \log D_i(f).
\end{aligned}
\end{equation*}
Since $D$ is a multi-class classifier, the above optimization has an implicit constraint $\sum_{i\in[M]} D_i(f) = 1$ for any $f$. Thanks to the infinite capacity of the discriminator $D$, we can maximize the value function pointwisely (for each $f$), and obtain a close form solution:
\begin{equation}\label{eqn:unconditionDA_Di_app_opt}
\begin{aligned}
    D_i^*(f) &= \alpha_{i} P(f| obj_i) / \sum_{i\in [M]} \alpha_{i} P(f| obj_i) \\
    &= \alpha_{i} P(f| obj_i)/P(f).
\end{aligned}
\end{equation}
The second equation is by definition~\eqref{eqn:advdadist_app}. The first equation can be obtained by the method of Lagrange multipliers, i.e.,
\begin{equation*}
\begin{aligned}
    D^* = \argmax_{D} &\sum_{i\in [M]} \alpha_{i} \int_f P(f| obj_i) \log D_i(f) \\
    & + \lambda(\sum_{i\in[M]} D_i(f) = 1).
\end{aligned}
\end{equation*}
Setting the derivative of the above equation w.r.t. $D_i(f)$ to zero, we obtain $D_i^*(f) = -\alpha_{i} P(f| obj_i)/\lambda$. Setting the derivative of $\lambda$ to zero, we obtain $\lambda = - \sum_{i\in[M]} \alpha_{i} P(f| obj_i)$. Therefore, we prove the optimal solution $D^*$ in \eqref{eqn:unconditionDA_Di_app_opt}.

Plugging \eqref{eqn:unconditionDA_Di_app_opt} into \eqref{eqn:unconditionDA_Di_app}, we have the maximum of \eqref{eqn:unconditionDA_Di_app} as
\begin{equation*}
\begin{aligned}
    & \sum_{i\in[M]} \alpha_{i} E_{f \in P(f_{u}| obj_i)}[\log \frac{\alpha_{i} P(f| obj_i)}{P(f)}] \\
    & = \sum_{i\in[M]} \alpha_i \log \alpha_i + \alpha_{i} KL(P(f_{u}| obj_i) || P(f)) \\
    & = KLD + \sum_{i\in[M]} \alpha_i \log \alpha_i.
\end{aligned}
\end{equation*}
Therefore, we proved that the maximum of \eqref{eqn:unconditionDA_Di_app} is the KLD plus a constant.
\end{proof}

\paragraph{From population loss to minibatch loss} The empirical loss of \eqref{eqn:unconditionDA_Di_app} is
\begin{equation}\label{eqn:unconditionDA_Dierm_app}
	\max_{D} \sum_{i\in [M]} \frac{\alpha_i}{N_i} \sum_{j=1}^{N_i} \log D_i(F(x_{j|i})),
\end{equation}
where $x_{j|i}$ denotes the $j$'th sample in domain $i$. In practice, we optimize this with SGD. For each sample $x$ (with domain $obj(x)$), we train $D$ with the following minibatch loss:
\begin{equation}\label{eqn:unconditionDA_Dierm_sgd_app}
\max_{D} \frac{N \alpha_i}{N_i} \log D_{obj(x)}(F(x)) = \log D_{obj(x)}(F(x)),
\end{equation}
where $\alpha_i = N_i/N$ (the recommended weight) is used in the last equality. Other kinds of weighting can be applied, too.
Finally, the feature extractor is updated by 
\begin{equation*}
\min_{F} \log D_{obj(x)}(F(x)),
\end{equation*}
as we depicted in \eqref{eqn:unconditionDA_Dierm_sgd} in Section~\ref{subsec:uncondADG}.

\begin{thm}\label{thm:cadg_kld}
Define the class-specific pooled feature distribution $P(f_{u}|pred_k)$ as
\begin{equation}\label{eqn:cadvdadist_app}
\begin{aligned}
P(f_{u}| pred_k) = \sum_{i\in [M]} \alpha_{i}^{(k)} P(f_{u}| obj_i, pred_k)
\end{aligned}
\end{equation} 
and $\alpha_{i}^{(k)}$ is the relative importance of each domain. We introduce discriminators specific to each class, i.e., $D(f; pred_k) \to [0,1]^{M}$ that tries to classify the domain (object category) based on the union-box feature and the predicate class. Mathematically, the discriminator is trying to maximize the likelihood of the ground-truth domain: 
\begin{equation}\label{eqn:cadvdakldiv_Di_app}
\begin{aligned}
& \max_{D} \sum_{k\in[K]} \alpha^{(k)} \sum_{i\in [M]} \alpha_{i}^{(k)} \\
& E_{f \in P(f_{u}| obj_i, pred_k)}[\log D_i(f; pred_k)] .
\end{aligned}
\end{equation}
Assume infinite capacity of the discriminator $D$, the maximum of \eqref{eqn:cadvdakldiv_Di_app} is a weighted summation of KL divergence between distributions (up to a constant):
\begin{equation}\label{eqn:cadvdakldiv_app}
\begin{aligned}
    CKLD :=& \sum_{k\in [K]} \alpha^{(k)} \sum_{i\in [M]} \alpha_{i}^{(k)} \\
    & KL(P(f_{u}| obj_i, pred_k) || P(f_{u}|pred_k)).
\end{aligned}
\end{equation} 
\end{thm}
The proof is nearly the same with the proof of Theorem~\eqref{thm:adg_kld}, so we omit the proof here.

\paragraph{From population loss to minibatch loss} The empirical loss of \eqref{eqn:cadvdakldiv_Di_app} is:
\begin{equation}\label{eqn:cadvdakldiv_Derm_app}
\max_{D} \sum_{k\in[K]} \sum_{i\in [M]} \frac{\alpha^{(k)} \alpha_{i}^{(k)}}{N_{i}^{(k)}} \sum_{j=1}^{N_{i}^{(k)}} \log D_i(f(x_{j|i,k}); pred_k),
\end{equation}
where $x_{j|i,k}$ denotes the $j$'th sample in domain $i$ with label $k$.
In this paper, we propose to use $\alpha^{(k)} = N^{(k)}/N$ and $\alpha_{i}^{(k)} = N_{i}^{(k)}/N^{(k)}$, so we have 
\begin{equation}\label{eqn:cadvdakldiv_Derm2}
\frac{1}{N}\max_{D} \sum_{k\in[K]} \sum_{i\in [M]} \sum_{j=1}^{N_{i}^{(k)}} \log D_i(f(x_{j|i,k}); pred_k) .
\end{equation}
In practice, we optimize this with SGD, in which for each sample $x$ (with label $pred(x)$ and domain $obj(x)$), we add the following regularization:
\begin{equation*}
\max_{D} \log D_{obj(x)}(f(x); pred(x)),
\end{equation*}
as we depicted in \eqref{eqn:cadvdakldiv_Derm_sgd} in Section~\ref{sec:cadvda}.

\begin{thm}\label{thm:cadg_jsd}
Define the class-specific pooled feature distribution $P(f_{u}|pred_k)$ as
\begin{equation}\label{eqn:cadvdadist_app2}
\begin{aligned}
P(f_{u}| pred_k) = \sum_{i\in [M]} \alpha_{i}^{(k)} P(f_{u}| obj_i, pred_k)
\end{aligned}
\end{equation} 
and $\alpha_{i}^{(k)}$ is the relative importance of each domain. We introduce discriminators specific to each class, i.e., $D(f; pred_k) \to [0,1]^{M}$ that tries to tell whether a feature $f$ is from the domain-conditioned distribution $P(f_{u}| obj_i, pred_k)$ or from the pooled distribution $P(f_{u}| pred_k)$. Mathematically, the discriminator (with binary output) is trying to maximize the following objective: 
\begin{equation}\label{eqn:cadvdajddiv_D_app}
\begin{aligned}
    & \max_{D} \sum_{k\in[K]} \alpha^{(k)} \sum_{i\in [M]} \alpha_{i}^{(k)} \\
    & \big( E_{f \in P(f_{u}| obj_i, pred_k)}[\log D(f, obj_i; pred_k)] \\
    & + E_{f \in P(f_{u}| pred_k)}[\log (1-D(f, obj_i; pred_k))] \big).
\end{aligned}
\end{equation}
Assume infinite capacity of the discriminator $D$, the maximum of \eqref{eqn:cadvdajddiv_D_app} is a weighted summation of Jensen-Shannon divergence (JSD) between distributions (up to a constant):
\begin{equation}\label{eqn:cadvdajsdiv_app}
\begin{aligned}
    CJSD :=& \sum_{k\in [K]} \alpha^{(k)} \sum_{i\in [M]} \alpha_{i}^{(k)} \\
    & JSD(P(f_{u}| obj_i, pred_k) || P(f_{u}|pred_k)).
\end{aligned}
\end{equation} 
\end{thm}
\begin{proof}
From the standard GAN proof, see, e.g., \cite{goodfellow2014generative}, assuming infinite capacity of the discriminator $D$, the maximum of \eqref{eqn:cadvdajddiv_D_app} is
\begin{equation*}
\begin{aligned}
    & \sum_{k\in[K]} \alpha^{(k)} \sum_{i\in [M]} \alpha_{i}^{(k)} \big( -\log(4) + \\
    & 2 * JSD(P(f_{ubox}| obj_i, pred_k) || P(f_{ubox}| pred_k)) \big) \\
    & = 2 * CJSD - log(4) \sum_{k\in[K]} \alpha^{(k)}.
\end{aligned}
\end{equation*}
\end{proof}

\paragraph{From population loss to minibatch loss} The empirical loss of \eqref{eqn:cadvdajddiv_D_app} is:
\begin{equation}\label{eqn:cadvdajddiv_Derm_app}
\begin{aligned}
    & \max_{D} \sum_{k\in[K]} \sum_{i\in [M]} \big( \frac{\alpha^{(k)} \alpha_{i}^{(k)}}{N_{i}^{(k)}} \sum_{j=1}^{N_{i}^{(k)}} \log D(f(x_{j|i,k}), obj_i; pred_k) \\
    & + \frac{\alpha^{(k)} \alpha_{i}^{(k)}}{N^{(k)}} \sum_{j=1}^{N^{(k)}} \log (1-D(f(x_{j|k}), obj_i; pred_k)) \big),
\end{aligned}
\end{equation}
where $x_{j|i,k}$ denotes the $j$'th sample in domain $i$ with label $k$ and $x_{j|k}$ denotes the $j$'th sample with label $k$.
In this paper, we propose to use $\alpha^{(k)} = N^{(k)}/N$ and $\alpha_{i}^{(k)} = N_{i}^{(k)}/N^{(k)}$, so we have 
\begin{equation}\label{eqn:cadvdajddiv_Derm2_app}
\begin{aligned}
    & \frac{1}{N}\max_{D} \sum_{k\in[K]} \sum_{i\in [M]} \big( \sum_{j=1}^{N_{i}^{(k)}} \log D(f(x_{j|i,k}), obj_i; pred_k) \\
    & + \frac{N_{i}^{(k)}}{N^{(k)}} \sum_{j=1}^{N^{(k)}} \log (1-D(f(x_{j|k}), obj_i; pred_k)) \big).
\end{aligned}
\end{equation}
In practice, we optimize this with SGD, in which for each sample $x$ (with label $pred(x)$ and domain $obj(x)$), we add the following regularization:
\begin{equation*}
\begin{aligned}
    & \max_{D} ~ \log D(f(x), obj(x); pred(x)) + \\
    & \sum_{i\in [O]} \frac{N_{i}^{pred(x)}}{N^{pred(x)}} \log (1-D(f(x), obj_i; pred(x))),
\end{aligned}
\end{equation*}
as we depicted in \eqref{eqn:cadvdajddiv_Derm_sgd} in Section~\ref{sec:cadvda}.

\section{Network architectures}
We present the network architectures for the baseline model and our proposed methods in Table~\ref{tab:detailed_architecture}, where the models are built with basic blocks defined in Table~\ref{tab:basic_blocks}. As our models take the ResNet-50 architecture as our backbone module, we only list the network structures after the backbone module. Although our models have three branches, we only apply the proposed methods on the union-box branch, so we only list the detailed structures of the union-box branch for each model. As shown in Table~\ref{tab:detailed_architecture}, the baseline model contains the blocks of ROI Align, $P_0$, $P_1$, $P_2$, and $P_3$. The adversarial branch of ADG-KLD takes $fc4$ as input and predicts the object categories. The conditional adversarial branch of CADG-KLD takes both $fc4$ and the predicate embedding $emb1$ as inputs and predicts the object categories. The conditional adversarial branch of CADG-JSD takes $fc4$, the predicate embedding $emb1$, and the object embedding $emb2$ as inputs, and provides a binary predication.

\begin{table*}[h]
\footnotesize
\centering
\begin{tabular}[t]{|p{3.3cm}|p{9.0cm}|}\hline
{Name} &{\makebox[8.3cm]{Operations / Layers}}\\ \hline\hline
{Conv $1 \times 1$, stride=$1$} &Convolution $1 \times 1$ - ReLU, stride=$1$.\\\hline
{Linear} &Fully connected layer.\\\hline
{FC} &Fully connected layer - ReLU.\\\hline
{Multiply} & Multiplication of two tensors (with broadcasting).\\\hline
{Mean} & Take the average of the tensor along the channel dimension.\\\hline
{Avg Pool} & Take the average of the tensor along the spatial dimensions.\\\hline
{Concat} &Concatenate input tensors along the channel dimension.  \\ \hline
{ResBlock} &Standard ResNet blocks.  \\ \hline
{ROI Align} &Pooling feature maps for ROI.  \\\hline
{Embedding} &Word embedding of the given words.  \\\hline
{$F_{backbone}$} & Features extracted from the backbone module with height $N_h$ and width $N_w$.  \\\hline
\end{tabular}
\caption{\footnotesize The basic blocks for architecture design. (``-" connects two consecutive layers.)}
\label{tab:basic_blocks}
\end{table*}

\begin{table*}[h]
\footnotesize
\centering
\begin{tabular}[t]{|c|p{2.0cm}|c|c|}\hline
{Stage} &{Name} &{\makebox[2.0cm]{Input Tensors}} &{\makebox[2.0cm]{Output Tensors}}\\ \hline\hline
\multirow{1}{*}{ROI Align} & ROI Align & $N_h \times N_w \times 1024 (F_{backbone})$ & $14 \times 14 \times 1024$ \\\hline\hline
\multirow{4}{*}{$P_0$} & ResBlock & $14 \times 14 \times 1024$ & $7 \times 7 \times 1024$ \\\cline{2-4}
& ResBlock & $7 \times 7 \times 1024$ & $7 \times 7 \times 1024$ \\\cline{2-4}
& ResBlock & $7 \times 7 \times 1024$ & $7 \times 7 \times 1024$ \\\cline{2-4}
& Avg Pool & $7 \times 7 \times 1024$ & $1024 (fc1)$  \\\hline\hline
\multirow{9}{*}{$P_1$} & FC & $1024 (fc1)$  & $512 (fc2)$ \\\cline{2-4}
& Conv $1 \times 1$, stride=$1$ & $N_h \times N_w \times 1024 (F_{backbone})$ & $N_h \times N_w \times 512 (conv1)$ \\\cline{2-4}
& Conv $1 \times 1$, stride=$1$ & $N_h \times N_w \times 1024 (F_{backbone})$ & $N_h \times N_w \times 512 (conv2)$ \\\cline{2-4}
& Multiply & $512(fc2), N_h \times N_w \times 512 (conv1)$ ,  & $N_h \times N_w \times 512$ \\\cline{2-4}
& Mean & $N_h \times N_w \times 512$ ,  & $N_h \times N_w$ \\\cline{2-4}
& Multiply & $N_h \times N_w, N_h \times N_w \times 512 (conv2)$ ,  & $N_h \times N_w \times 512$ \\\cline{2-4}
& Conv $1 \times 1$, stride=$1$ & $N_h \times N_w \times 512$ & $N_h \times N_w \times 1024$ \\\cline{2-4}
& Avg Pool & $N_h \times N_w \times 1024$ & $1024 (fc3)$ \\\cline{2-4}
& Concat & $1024 (fc1), 1024 (fc3)$ & $2048$ \\\hline\hline
\multirow{2}{*}{$P_2$} & FC & $2048$  & $1024$ \\\cline{2-4}
& FC & $1024$  & $1024 (fc4)$ \\\hline\hline
\multirow{1}{*}{$P_3$} & Linear & $1024$  & $117$ \\\hline\hline
\multirow{2}{*}{$Emb$} & Embedding & $1$ (predicate label)  & $50 (emb1)$ \\\cline{2-4}
& Embedding & $1$ (object label)  & $50 (emb2)$ \\\hline\hline
\multirow{1}{*}{$ADG-KLD$} & Linear & $1024 (fc4)$  & $80$ \\\hline\hline
\multirow{2}{*}{$CADG-KLD$} & Concat & $1024 (fc4), 50 (emb1)$  & $1074$ \\\cline{2-4}
& Linear & $1074$  & $80$ \\\hline\hline
\multirow{2}{*}{$CADG-JSD$} & Concat & $1024 (fc4), 50 (emb1), 50 (emb2)$  & $1124$ \\\cline{2-4}
& Linear & $1124$  & $1$ \\\hline

\end{tabular}
\caption{\footnotesize The structures for the baseline model and our proposed methods ADG-KLD, CADG-KLD, and CADG-JSD..}
\label{tab:detailed_architecture}
\end{table*}

\noindent\textbf{Implementation details}
For the input images in both training and inference, we first resize them to set the larger size at 640. During training, we apply gradient clip to 1 when doing back-propagation and use early stopping to train the model with 4 NVIDIA P100 GPU, where we adjust the learning schedule according to the linear scaling rule as in \cite{goyal2017accurate}. Because we use the sigmoid loss function for each predicate category and the ratio of positive and negative samples are imbalanced, we set the ratio of positive and negative samples to 1:6. We've made hyper-parameter search on this ratio and find that the performance doesn't change a lot for different ratios, but it performs much better than the model without setting this ratio. 

After we make extensive experiments on the hyper-parameter search, we employ a SGD optimizer with learning rate 0.001, momentum 0.9 and weight decay 0.0005. For the proposed domain generalization framework, we train the mainstream branch and the adversarial branch each for once alternatively. 

\noindent\textbf{Baseline}
We set the learning rate decay at every 100,000 iterations with $gamma=0.96$. 

\noindent\textbf{ADG-KLD}
We set the learning rate decay at the 2,000,000-th iteration and 2,800,000-th iteration with $gamma=0.1$ and set $\lambda_{DG}=1$. 

\noindent\textbf{CADG-KLD}
We set the learning rate decay at the 2,000,000-th iteration and 2,800,000-th iteration with $gamma=0.1$ and set $\lambda_{DG}=100$. 

\noindent\textbf{CADG-JSD}
We set the learning rate decay at every 100,000 iterations with $gamma=0.96$ and set $\lambda_{DG}=100$. The learning rate for the adversarial branch is 0.01.

\section{Grad-CAM visualization}
We demonstrate more Grad-CAM visualization of intermediate features in Fig.~\ref{fig:supp_gradcam_O_branch} and Fig.~\ref{fig:supp_gradcam_backbone}. Fig.~\ref{fig:supp_gradcam_O_branch} visualizes the feature maps from the union-box branch, where the baseline attends to the wrong region and our proposed CADG-KLD and CADG-JSD focus on the region related to the human-object interaction. ADG-KLD also cannot give correct saliency maps on the union-box features, which is the same as we point out in the main paper. Fig.~\ref{fig:supp_gradcam_backbone} shows the feature maps from the backbone features, where our proposed methods show more meaningful saliency maps than the baseline. Furthermore, CADG-KLD and CADG-JSD can get even better attention on the possible interactions than ADG-KLD. Therefore, we can safely make the same conclusion as in the main paper that our proposed methods could learn semantically rich features of predicates with strong generalization ability, and the conditional domain generalization approaches can get better saliency on the union-box features.

\begin{figure*}[ht]
\centering
\small
\setlength{\tabcolsep}{1pt}
\begin{tabular}{ccccc}
 \includegraphics[width=.18\textwidth]{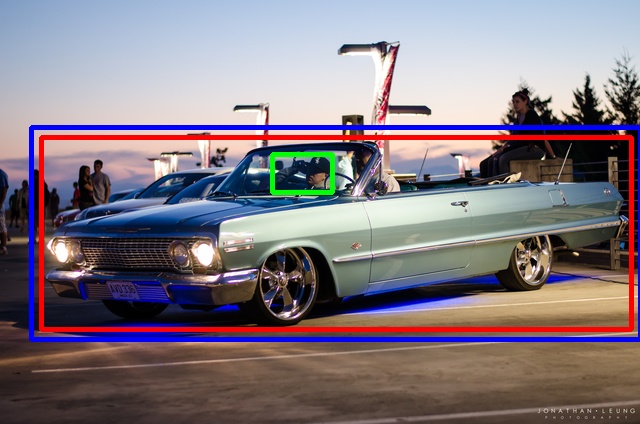}&
 \includegraphics[width=.18\textwidth]{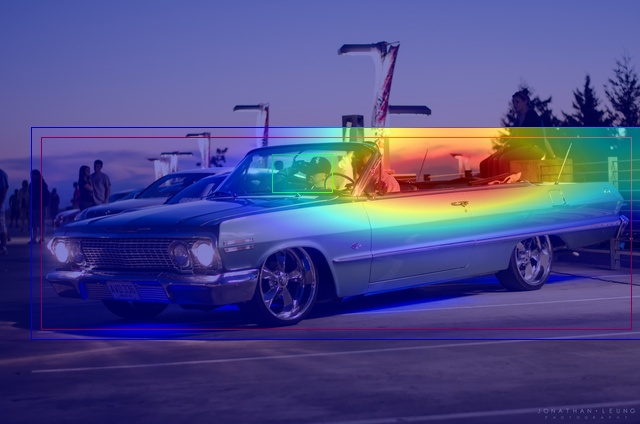}&
 \includegraphics[width=.18\textwidth]{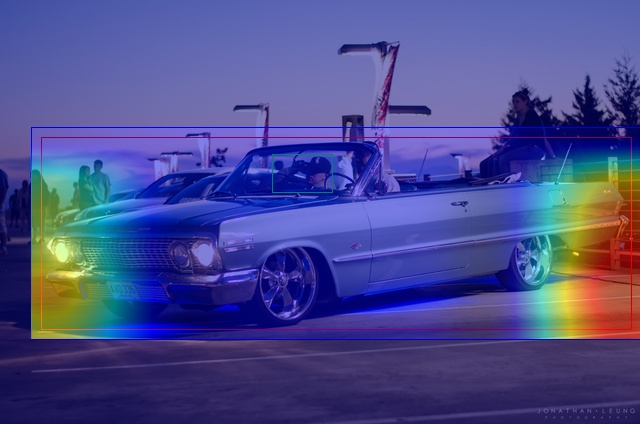}&
 \includegraphics[width=.18\textwidth]{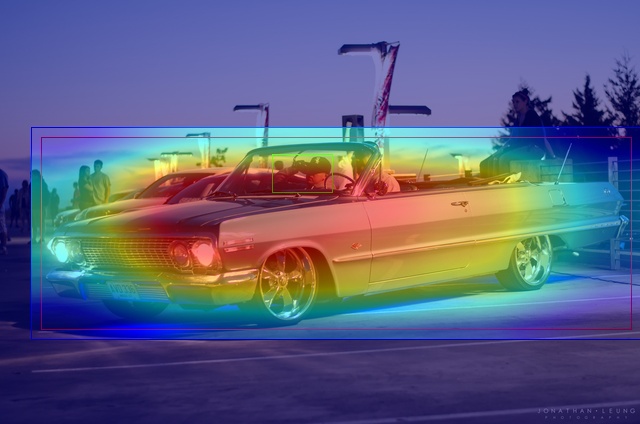}&
 \includegraphics[width=.18\textwidth]{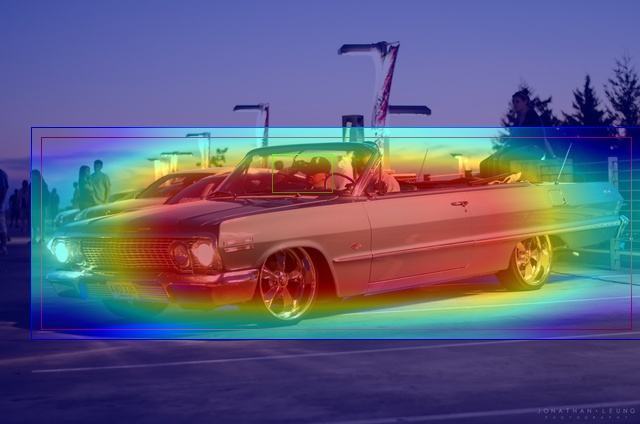}\\
 \includegraphics[width=.18\textwidth]{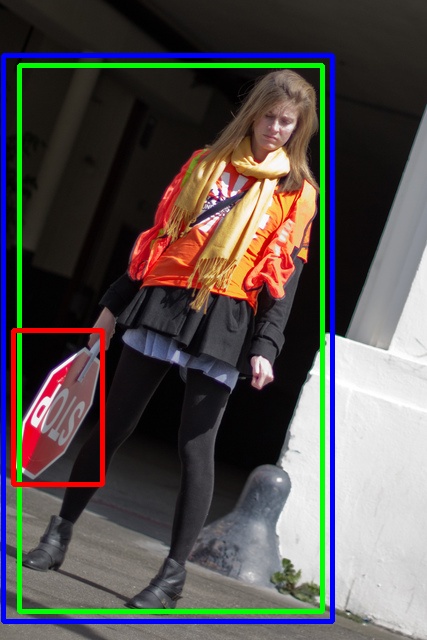}&
 \includegraphics[width=.18\textwidth]{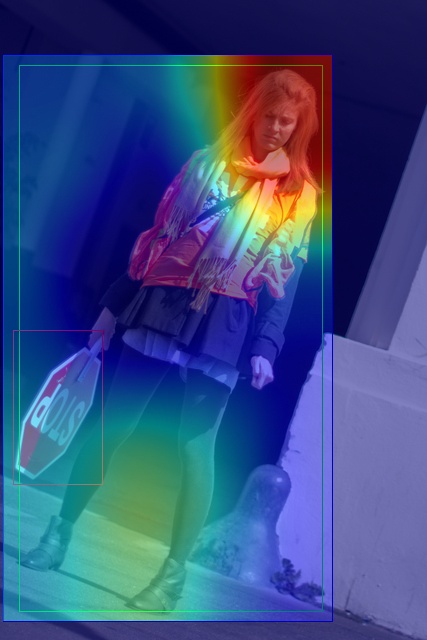}&
 \includegraphics[width=.18\textwidth]{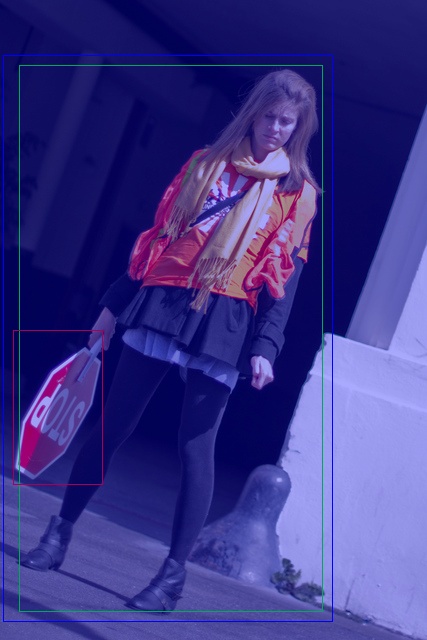}&
 \includegraphics[width=.18\textwidth]{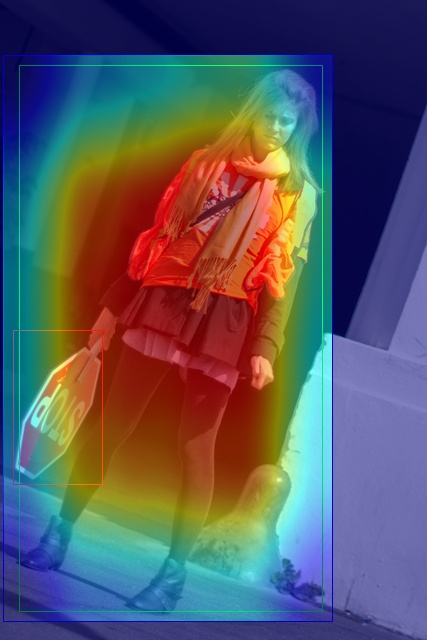}&
 \includegraphics[width=.18\textwidth]{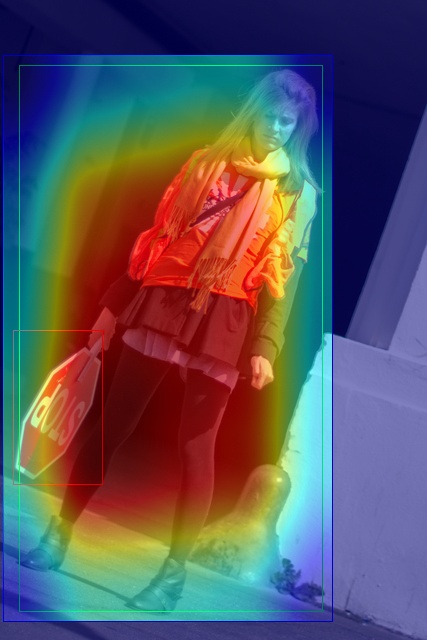}\\
 \includegraphics[width=.18\textwidth]{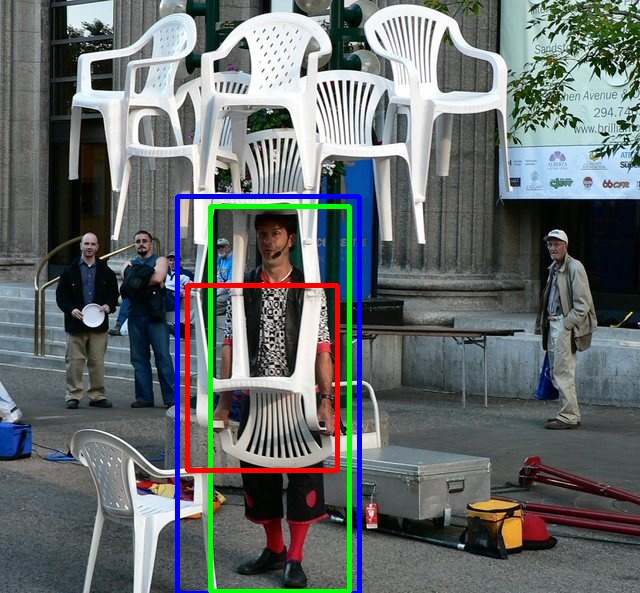}&
 \includegraphics[width=.18\textwidth]{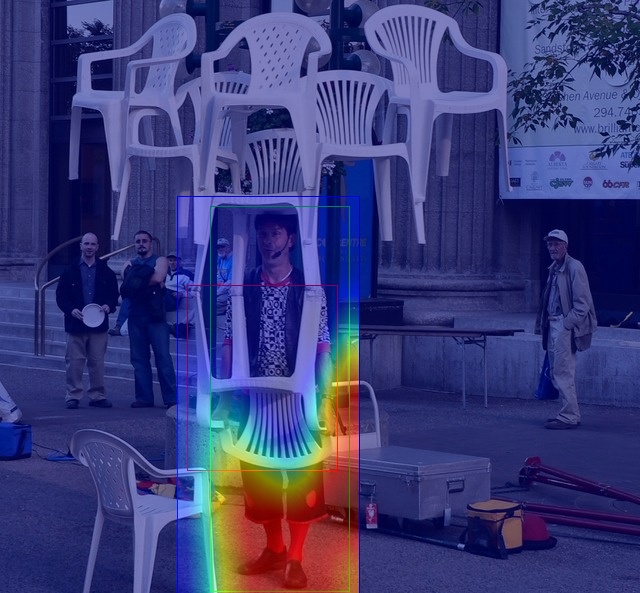}&
 \includegraphics[width=.18\textwidth]{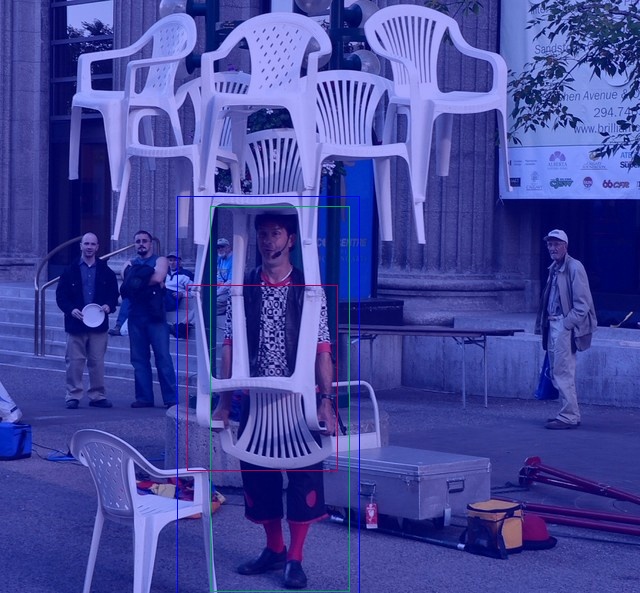}&
 \includegraphics[width=.18\textwidth]{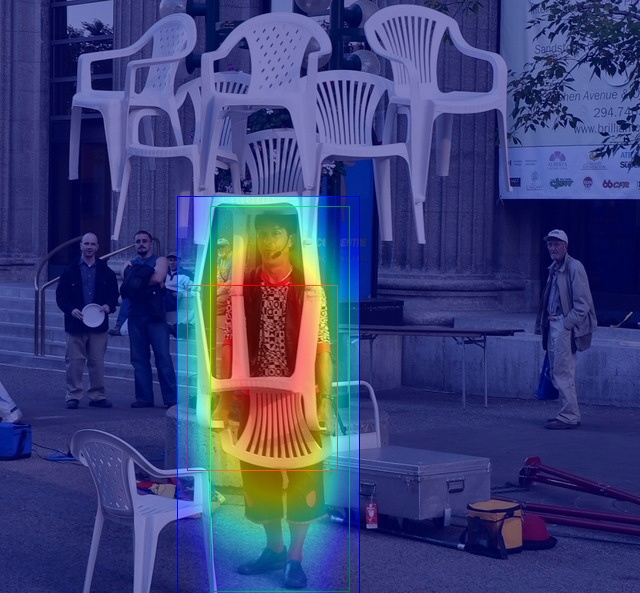}&
 \includegraphics[width=.18\textwidth]{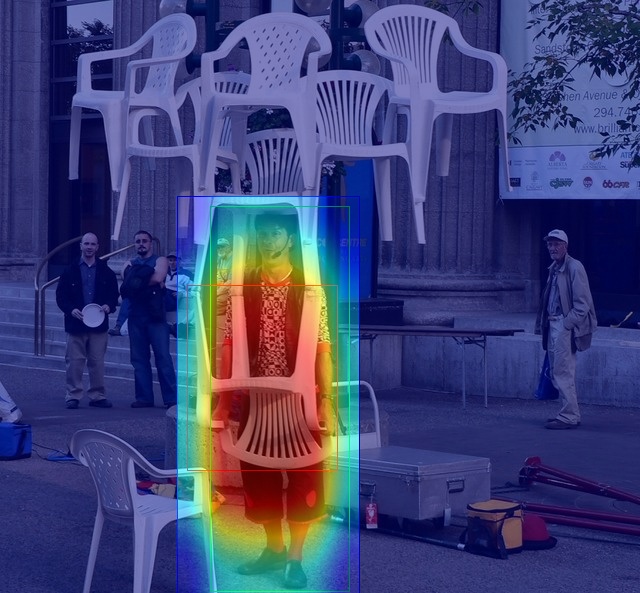}\\
 \includegraphics[width=.18\textwidth]{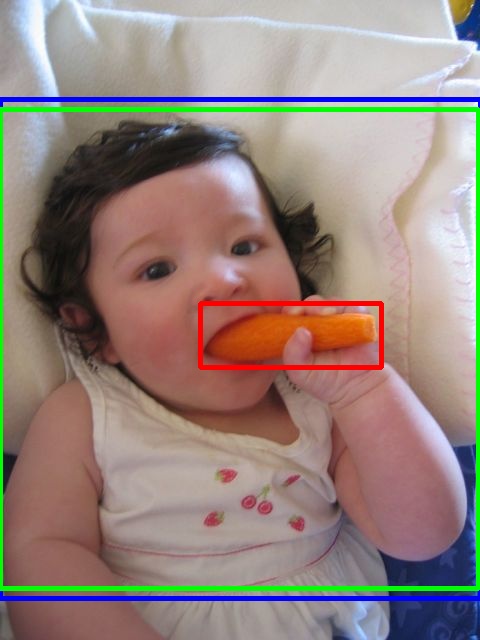}&
 \includegraphics[width=.18\textwidth]{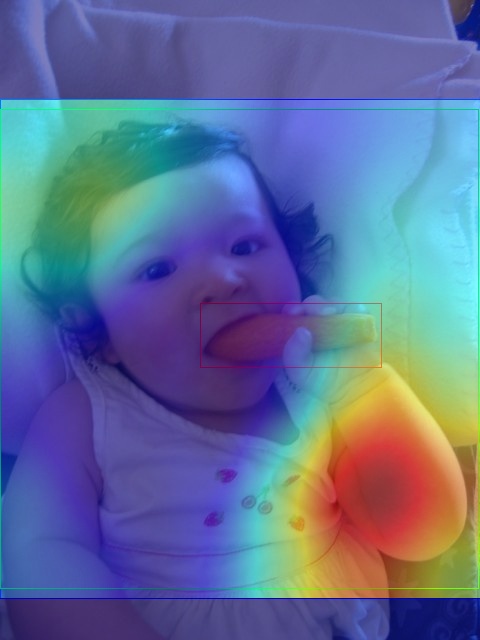}&
 \includegraphics[width=.18\textwidth]{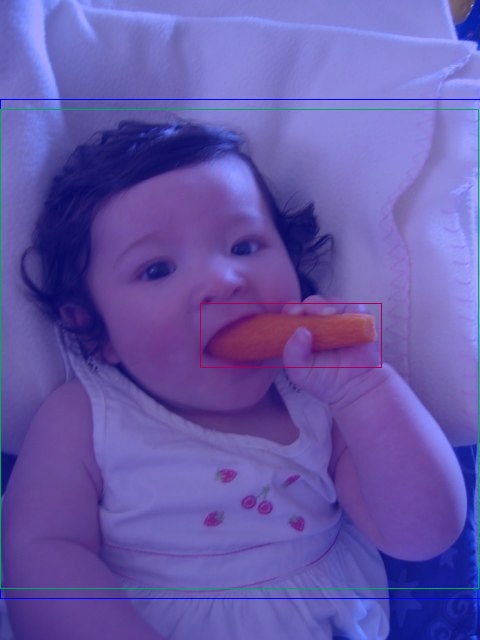}&
 \includegraphics[width=.18\textwidth]{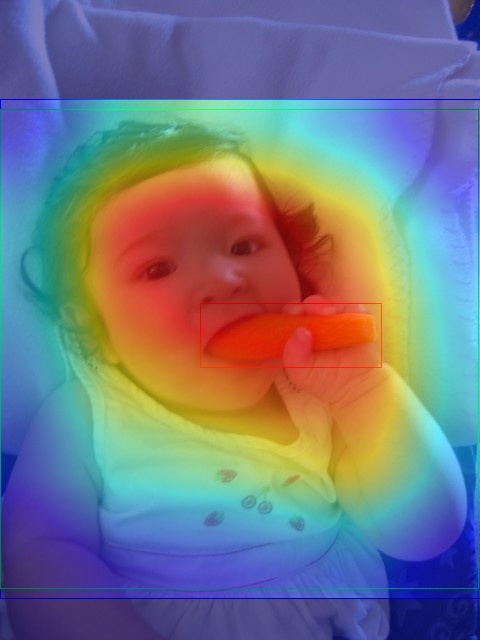}&
 \includegraphics[width=.18\textwidth]{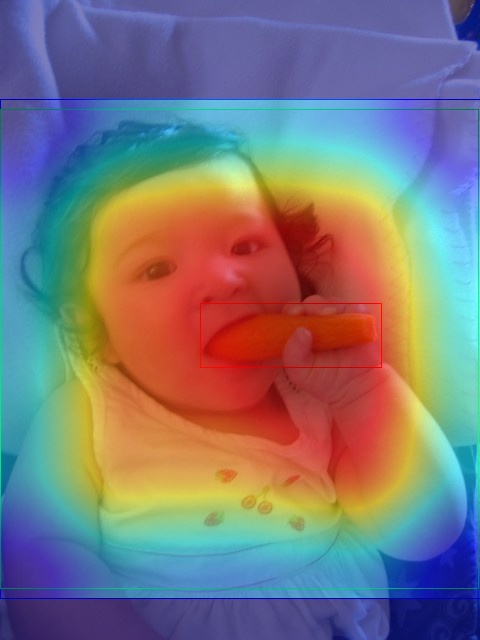}\\
(a) & (b) & (c) & (d) & (e) \\
\end{tabular}
\vspace{-8pt}
\caption{Grad-CAM visualization of the predicates from the \textbf{union-box} features. \textbf{Green box:} human. \textbf{Red box:} object. \textbf{Blue box:} union box of object and human with a margin. (a) input images. (b) baseline. (c) ADG-KLD. (d) CADG-KLD. (e) CADG-JSD. Zoom in for better view.}
\vspace{-5pt}
\label{fig:supp_gradcam_O_branch}
\end{figure*}

\begin{figure*}[ht]
\centering
\small
\setlength{\tabcolsep}{1pt}
\begin{tabular}{ccccc}
 \includegraphics[width=.18\textwidth]{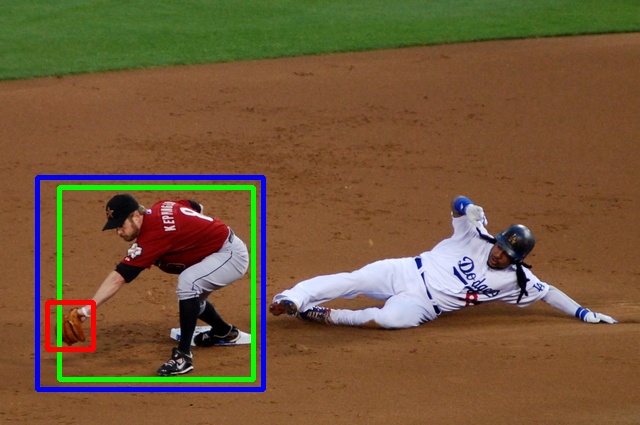}&
 \includegraphics[width=.18\textwidth]{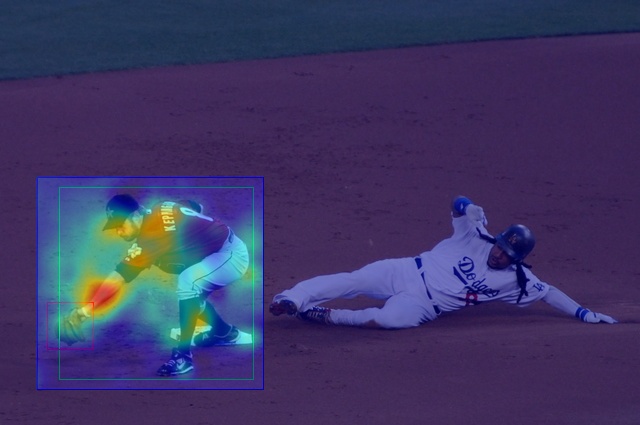}&
 \includegraphics[width=.18\textwidth]{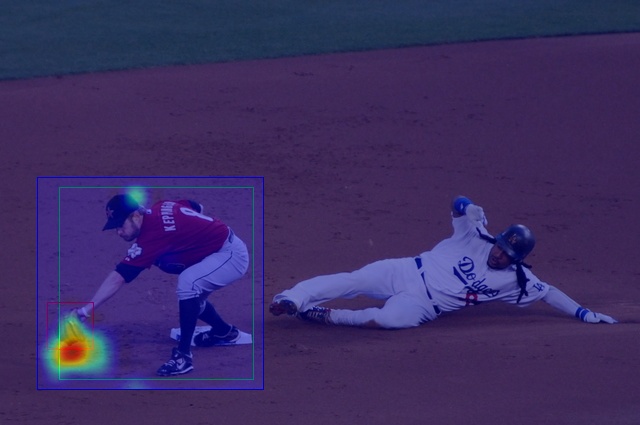}&
 \includegraphics[width=.18\textwidth]{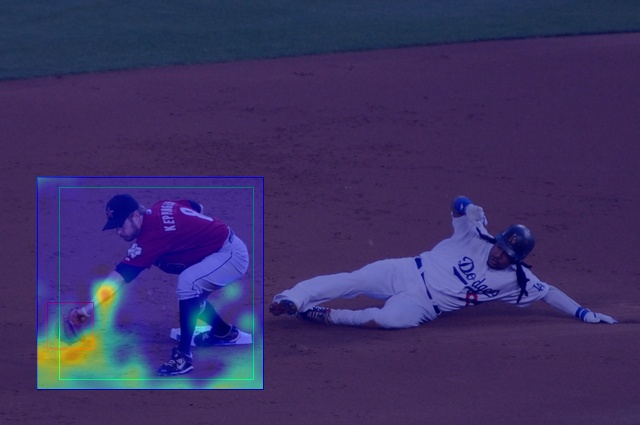}&
 \includegraphics[width=.18\textwidth]{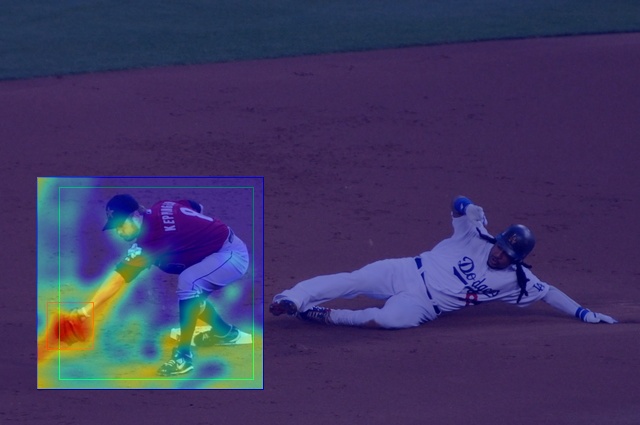}\\
 \includegraphics[width=.18\textwidth]{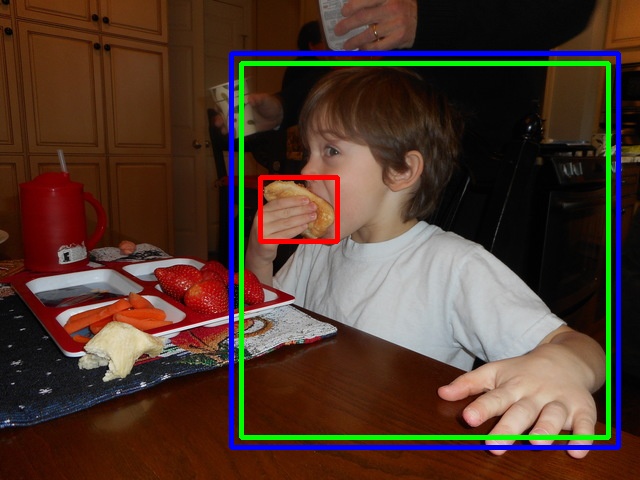}&
 \includegraphics[width=.18\textwidth]{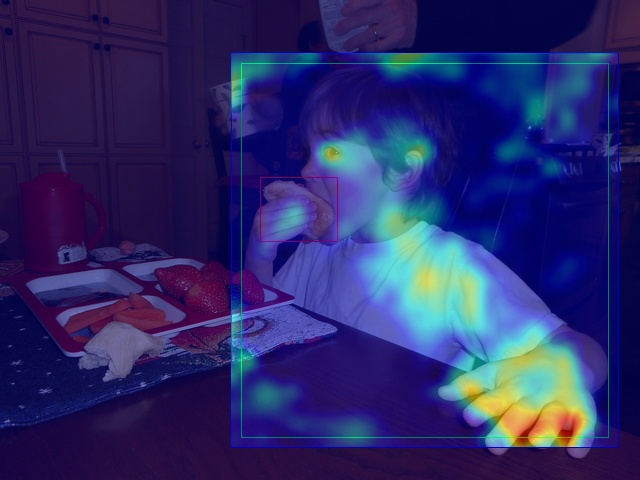}&
 \includegraphics[width=.18\textwidth]{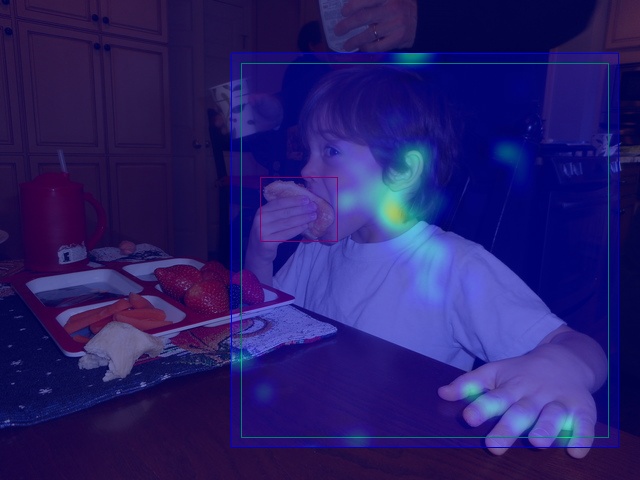}&
 \includegraphics[width=.18\textwidth]{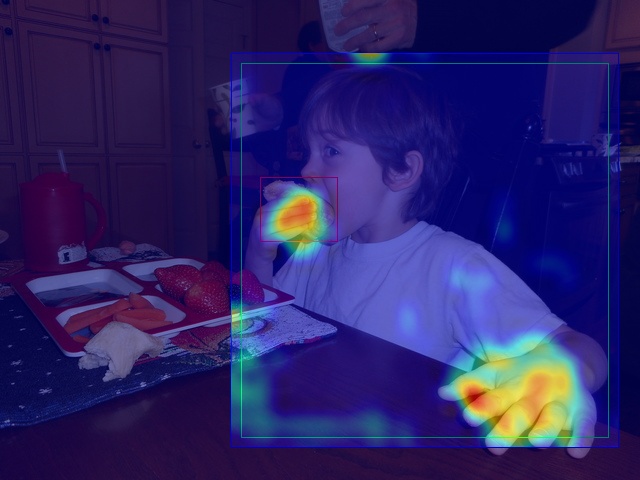}&
 \includegraphics[width=.18\textwidth]{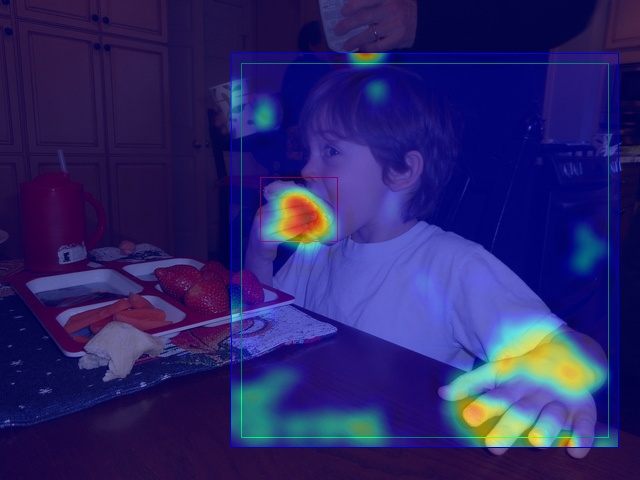}\\
 \includegraphics[width=.18\textwidth]{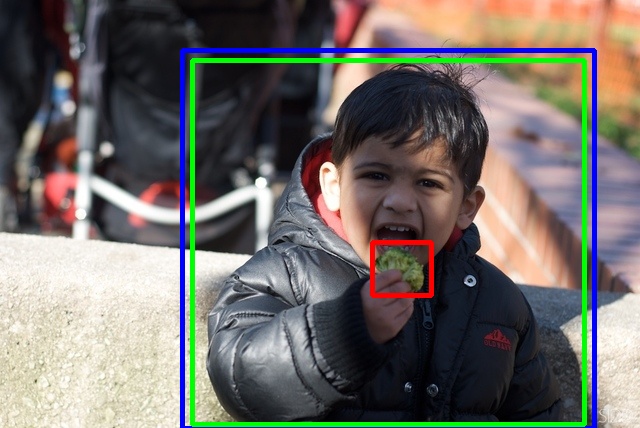}&
 \includegraphics[width=.18\textwidth]{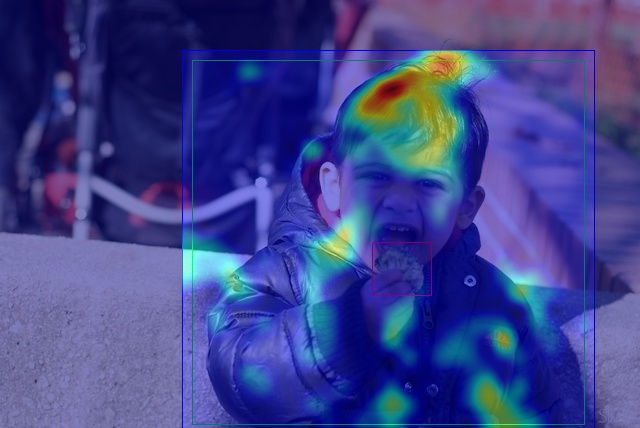}&
 \includegraphics[width=.18\textwidth]{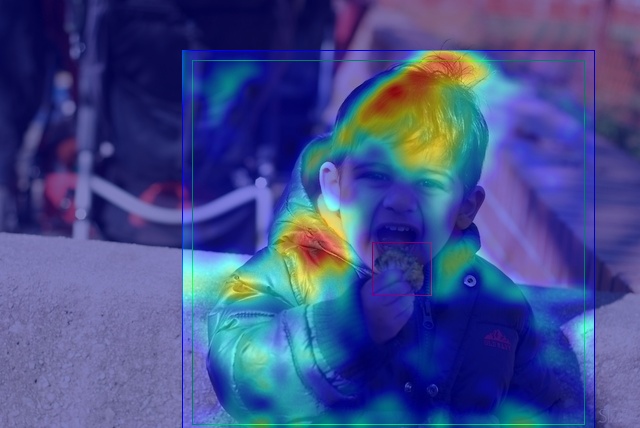}&
 \includegraphics[width=.18\textwidth]{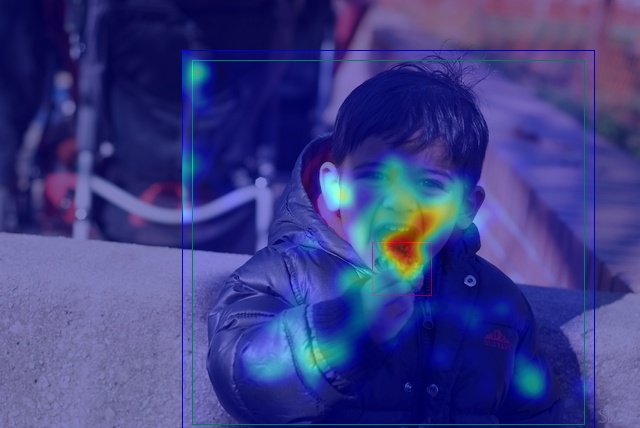}&
 \includegraphics[width=.18\textwidth]{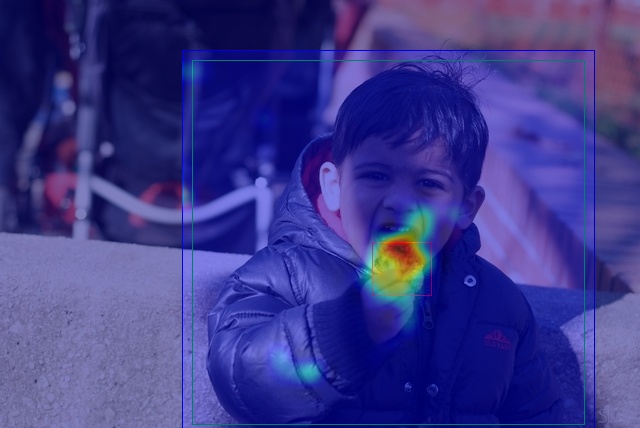}\\
 \includegraphics[width=.18\textwidth]{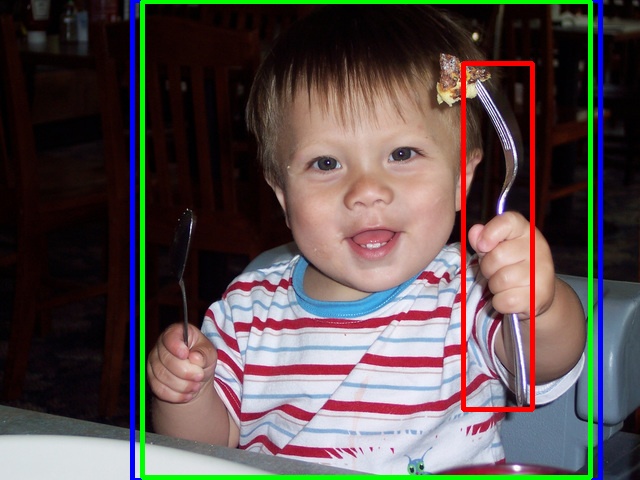}&
 \includegraphics[width=.18\textwidth]{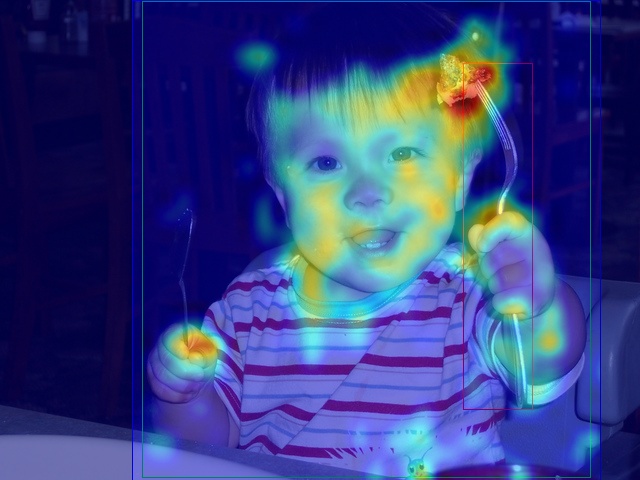}&
 \includegraphics[width=.18\textwidth]{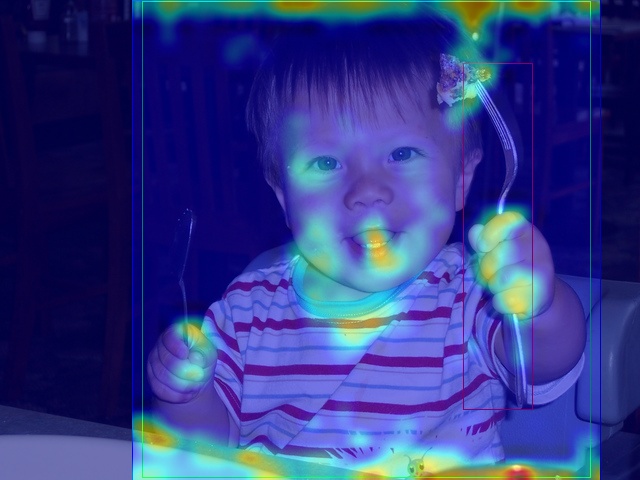}&
 \includegraphics[width=.18\textwidth]{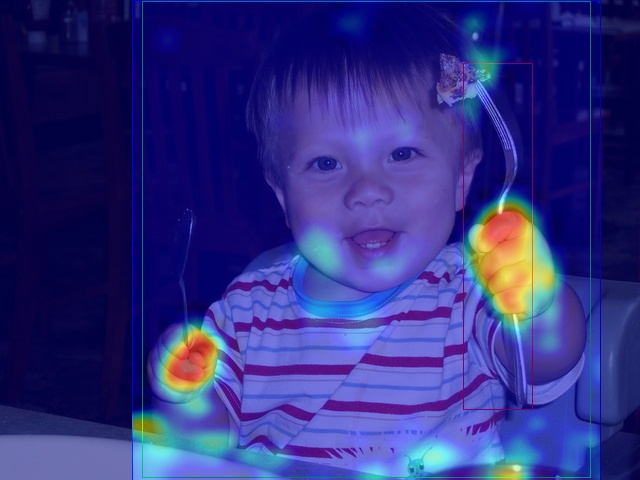}&
 \includegraphics[width=.18\textwidth]{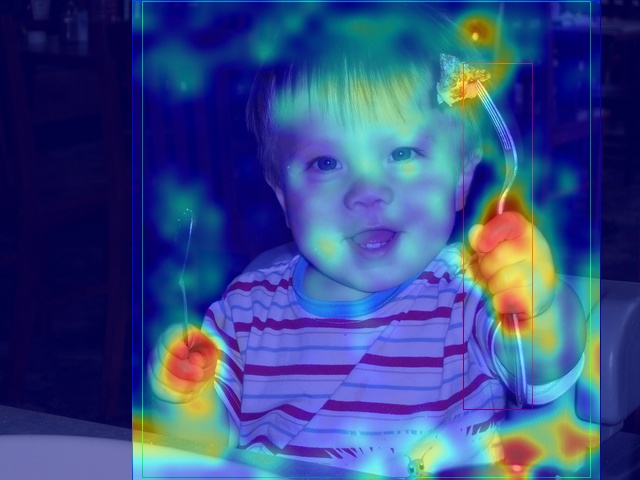}\\
(a) & (b) & (c) & (d) & (e) \\
\end{tabular}
\vspace{-8pt}
\caption{Grad-CAM visualization of the predicates from the \textbf{backbone} features before the ROI Align module (we only keep the visualization inside the union box for simplification). \textbf{Green box:} human. \textbf{Red box:} object. \textbf{Blue box:} union box of object and human with a margin. (a) input images. (b) baseline. (c) ADG-KLD. (d) CADG-KLD. (e) CADG-JSD. Zoom in for better view.}
\vspace{-5pt}
\label{fig:supp_gradcam_backbone}
\end{figure*}

\end{document}